\PassOptionsToPackage{table}{xcolor}
\documentclass[11pt,a4paper,logo]{lumia}

\usepackage[numbers,sort&compress]{natbib}
\usepackage{etoolbox}
\usepackage{algorithm}
\usepackage{algorithmic}
\usepackage{mathtools}
\usepackage{nicefrac}
\usepackage{bm}
\usepackage{multirow}
\usepackage{threeparttable}
\usepackage{subcaption}

\graphicspath{{./}}
\hypersetup{colorlinks=true, linkcolor=blue!55!black, citecolor=teal!60!black, urlcolor=blue!55!black}

\definecolor{abstrabg}{HTML}{F2D1C5}
\definecolor{black50}{gray}{0.5}

\definecolor{citecolor}{HTML}{0071bc}

\hypersetup{
    colorlinks=true,
    linkcolor=red,
    citecolor=citecolor,
    filecolor=magenta,
    urlcolor=magenta,
    pdftitle={Improving Generalization Robustness of Multimodal RLVR},
    pdfauthor={Pengfei Zhou, Zhiwei Tang, Chenrui Zhou, Lama Moukheiber, Yixing Ma, Bin Xu, Jiajun Song, Zhenglin Wan, Wangbo Zhao, Jiasheng Tang, Bohan Zhuang, Fan Wang, Yang You}
}

\newcommand{\suppsec}[1]{Appendix~#1}
\newcommand{\suppthm}[1]{Theorem~#1 in the Appendix}
\newcommand{\suppprop}[1]{Proposition~#1 in the Appendix}
\newcommand{\suppdef}[1]{Definition~#1 in the Appendix}

\newcommand{\suppalg}[1]{Algorithm~#1 in the Appendix}

\newcommand{\mainsec}[1]{Section~#1 of the main paper}
\newcommand{\maineq}[1]{Eq.~(#1) in the main paper}
\newcommand{\maintable}[1]{Table~#1 in the main paper}

\providecommand{\llbracket}{\mathopen{[\![}}
\providecommand{\rrbracket}{\mathclose{]\!]}}

\definecolor{lightgray}{RGB}{240,240,240}
\definecolor{darkgreen}{RGB}{0,150,90}
\definecolor{darkred}{RGB}{220,40,40}

\newcommand{\subindent}{\hspace{0.9em}}
\newcommand{\child}[1]{\subindent{\fontsize{7.8}{9}\selectfont #1}}

\theoremstyle{plain}
\newtheorem{theorem}{Theorem}[section]
\newtheorem{proposition}[theorem]{Proposition}
\newtheorem{lemma}[theorem]{Lemma}

\theoremstyle{definition}
\newtheorem{definition}[theorem]{Definition}
\newtheorem{assumption}[theorem]{Assumption}
\theoremstyle{remark}

\definecolor{myBlue}{RGB}{91,135,202}
\definecolor{myRed}{RGB}{196,93,93}
\definecolor{myGreen}{RGB}{94,171,105}

\correspondingemail{
\emailicon~\href{mailto:wangbo.zhao96@gmail.com}{wangbo.zhao96@gmail.com};
\emailicon~\href{mailto:jiasheng.tjs@alibaba-inc.com}{jiasheng.tjs@alibaba-inc.com}; \emailicon~\href{mailto:yangyou@nus.edu.sg}{yangyou@nus.edu.sg}
\quad $^\ddagger$ Corresponding Author\\[3pt]
\textbf{Project Page:}~\href{https://rlvr-pirl.vercel.app}{\nolinkurl{https://rlvr-pirl.vercel.app}}
}

\title{Improving Generalization Robustness of Multimodal RLVR}
\setheadertitle{Improving Generalization Robustness of Multimodal RLVR}

\author{%
    \parbox{0.75\textwidth}{\centering
    Pengfei~Zhou$^{1}$ \quad
    Zhiwei~Tang$^{2,3,4}$ \quad
    Chenrui~Zhou$^{1}$ \quad
    Lama~Moukheiber$^{6}$ \quad
    Yixing~Ma$^{5}$ \quad
    Bin~Xu$^{7}$ \quad
    Jiajun~Song$^{8}$ \quad
    Zhenglin~Wan$^{1}$ \quad
    Wangbo~Zhao$^{9}$$^\ddagger$ \quad
    Jiasheng~Tang$^{2,3}$$^\ddagger$ \quad
    Bohan~Zhuang$^{4}$ \quad
    Fan~Wang$^{2}$ \quad
    Yang~You$^{1}$$^\ddagger$ \\
    }\\
    \vspace{-0.5em}
    $^{1}$National University of Singapore, HPC-AI Lab \quad
    $^{2}$DAMO Academy, Alibaba Group \quad
    $^{3}$Hupan Lab \quad
    $^{4}$Zhejiang University \quad
    $^{5}$University of California, Berkeley \quad
    $^{6}$Georgia Institute of Technology \\
    $^{7}$InfRec, Cardinal AI Lab \quad
    $^{8}$Renmin University of China \quad
    $^{9}$Hong Kong University of Science and Technology \\
}

\begin{document}

\begin{abstract}
Reinforcement Learning with Verifiable Rewards (RLVR) makes Multimodal Large Language Models more accurate, but the gains are brittle: simply paraphrasing a question or changing the prompt template can degrade them, which challenges reliable deployment in high-stakes scenarios like medical VQA. We trace this to two issues of the standard RL objective. First, the binary verifier conflates format with content, so the reward signal cannot tell a wrong answer apart from a misformatted one. Second, the training distribution covers only a thin slice of the real-world prompts that the model might meet at deployment, so policies that perform well on the training distribution can behave differently under unseen prompts during test. Both failures call for a robust post-training method that helps the policy cover a broader distribution of semantically equivalent prompts, and we identify two measures that help achieve this objective: separating format from semantics in the reward, and applying policy invariance across perturbed prompts with equivalent semantics. We therefore propose \textbf{Prompt-Invariant RLVR (PIRL)}, consisting of a dynamic trinary reward and a consistency regularizer based on an embedding-space adversary. Under stress testing, PIRL's average accuracy on benchmarks drops by only $\le 1\%$, where GRPO drops $\sim 3\%$. On dynamic evaluation, PIRL also achieves the smallest performance drop.
\end{abstract}

\maketitle

\section{Introduction}
\label{sec:intro}

Reinforcement Learning with Verifiable Rewards (RLVR) drives substantial gains on visual question-answering (VQA) benchmarks for large language models and Multimodal Large Language Models (MLLMs)~\citep{shao2024deepseekmath, zhang2025survey, meng2025mm, yue2025does, zhou2025reinforced}, optimizing policies against deterministic signals such as symbolic verification or template matching. The prevailing assumption is that such rule-based reward functions incentivize generalizable reasoning~\citep{guo2025deepseek, zhou2026mdk12}.

\begin{figure}[ht]
    \centering
    \includegraphics[width=0.7\textwidth]{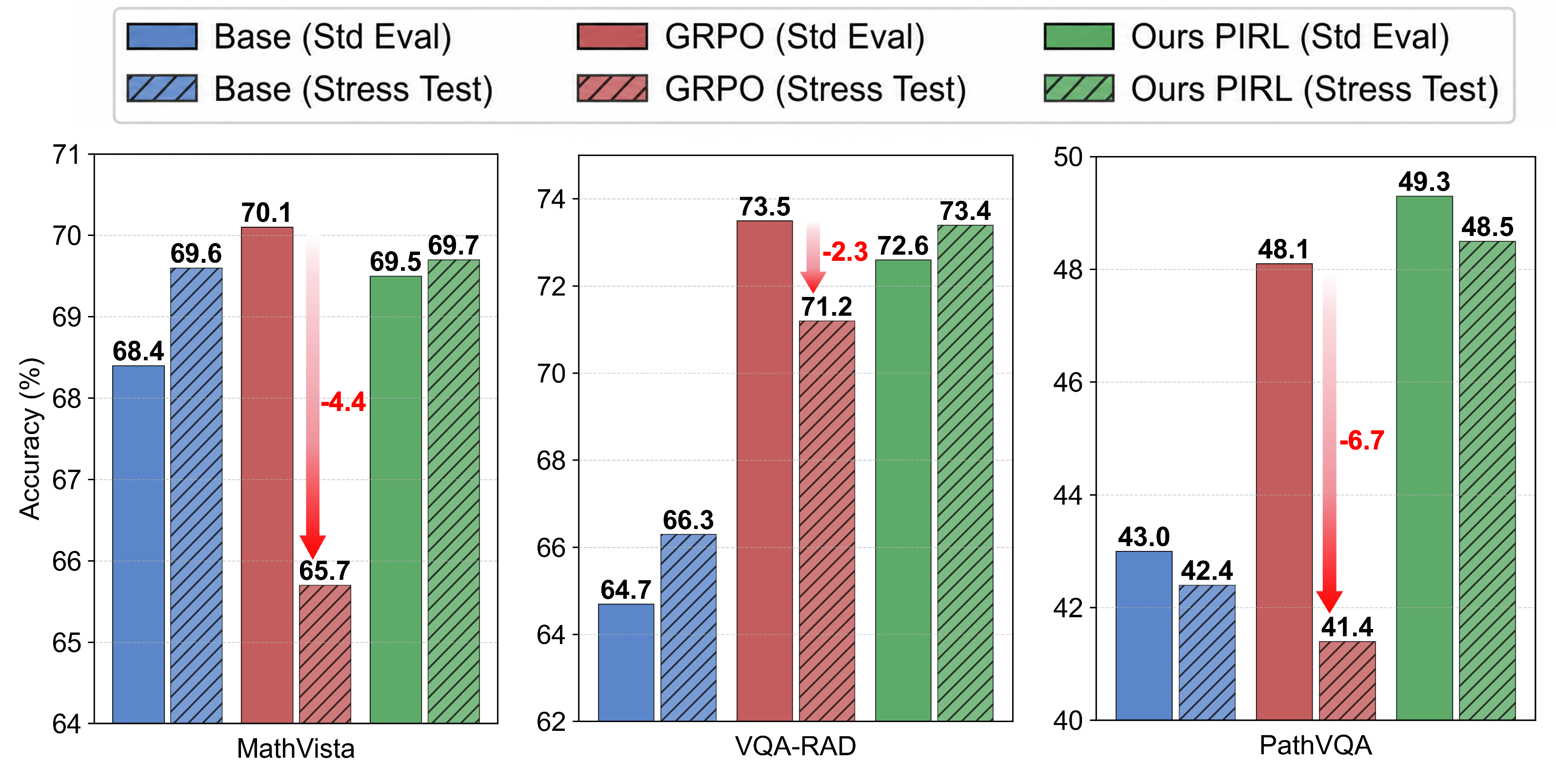}
    \caption{Standard RLVR (\textcolor{myRed}{red}) outperforms the base
      model (\textcolor{myBlue}{blue}) on training-format prompts but
      degrades sharply under transformed prompts (Stress Test);
      \textbf{PIRL} (\textcolor{myGreen}{\textbf{green}}) preserves accuracy under Stress Test.}
    \label{fig:intro}
\end{figure}

It fails under prompt perturbation. Fig.~\ref{fig:intro} shows that an RL-tuned MLLM beats its base model on the training prompt but degrades sharply when the same question is rephrased into a different template. We call this the \emph{prompt-robustness gap}: a model that scores well on the training-template prompt fails on a paraphrased restatement of the same question. This gap matters most where it is least affordable: in high-stakes domains such as medical and legal VQA, a model that degrades under perturbation is unreliable in deployment regardless of its accuracy. The open question is whether current RLVR can learn prompt-invariant reasoning.

The effect is amplified in the multimodal setting, where the visual side is high-dimensional and continuous while the format side is discrete and sparse. Optimization takes the easier route: exploit the textual format, leave visual understanding underdeveloped. This can be attributed to two issues.

\textbf{Reward Entanglement.}
The widely-used rule-based binary rewards (e.g., regex matching) conflate \emph{semantic correctness} with \emph{format compliance} via a multiplicative verifier $V = V_f \cdot V_c$. An obtained zero reward cannot distinguish multimodal reasoning failure and format non-compliance. We formalize the resulting loss of semantic reward variation in \suppthm{B.1}, which serves as a reward-level proxy for the gradient bias.

\textbf{Prompt Dependency.}
Standard RLVR optimizes over a fixed distribution of training prompts, while deployment prompts may follow different formats. We formalize this as a quantitative non-identifiability result (\suppprop{B.5}): the training objective can assign the same training reward to policies that differ on separated unseen formats, allowing a test-time reward gap proportional to the test mass on those formats.

Both failures share a single root: the training objective sees only the training distribution $P_{\text{train}}$ (Section~\ref{sec:preliminaries}) and cannot certify behavior on unobserved but semantics-preserving prompts. Therefore, we use a worst-case objective over the \emph{semantic equivalence class} $\llbracket\bm{x}\rrbracket$: transformations that preserve task semantics are applied regardless of divergence from $P_{\text{train}}$ (Section~\ref{sec:theory_dro}). For this objective we identify two design conditions, \textbf{(C1) Reward Decomposition} and \textbf{(C2) Invariance Control}, with explicit population-level and coverage-dependent consequences. Neither condition alone guarantees correctness or optimization convergence. Our proposed \textbf{Prompt-Invariant RLVR (PIRL)} instantiates (C1) with a Dynamic Trinary Reward and uses an embedding-space adversary plus consistency regularization as a practical surrogate for (C2).

Our contributions are summarized as follows:
\begin{itemize}
  \item \textbf{We reveal prompt-format overfitting as an issue of RLVR.} We identify the \emph{prompt-robustness gap}: RLVR-tuned MLLMs can improve on training-format prompts while degrading under semantics-preserving rephrasing. We show that this gap is rooted in the standard RLVR objective: binary rewards entangle format compliance with semantic correctness, and optimization over $P_{\text{train}}$ cannot distinguish prompt-invariant reasoning from template-specific behavior.
  \item \textbf{We formulate prompt robustness as optimization over semantic equivalence classes.} Rather than treating the observed prompt template as the unit of training, we define robustness over the semantic equivalence class $\llbracket \bm{x} \rrbracket$. This view yields two design principles: \textbf{Reward Decomposition}, which restores reward variation within format buckets under a non-degeneracy assumption, and \textbf{Invariance Control}, which bounds coverage-averaged answer variation across semantics-preserving prompt transformations.
  \item \textbf{We operationalize the principle with Prompt-Invariant RLVR.}
  PIRL applies a Dynamic Trinary Reward and an embedding-space prompt adversary with consistency regularization. The reward separates format compliance from semantic correctness, while the adversary perturbs the prompt embedding within a bounded neighborhood to encourage prompt-invariant behavior.
  \item \textbf{We evaluate PIRL against template augmentation.}
  Across exam, medical, and legal VQA on Qwen2.5-VL-7B and Qwen3-VL-8B, GRPO often improves in-distribution accuracy but induces larger mean drops under stress testing. PIRL substantially reduces these gaps relative to GRPO and has numerically smaller mean degradation than the multi-template GRPO baseline under both template-stress and dynamic evaluation; the component ablation does not establish that DTR or the adversary is individually necessary beyond MT.
\end{itemize}


\section{Theoretical Analysis of the Robustness Gap}
\label{sec:theory}

\subsection{Preliminary}
\label{sec:preliminaries}

An MLLM policy $\pi_\theta(\bm{y}\mid\bm{x})$ generates a response $\bm{y}$
conditioned on a prompt $\bm{x} = (\bm{v}, \bm{q}, \mathcal{I})$, where
$\bm{v}$ is a visual input, $\bm{q}$ is a question, and $\mathcal{I}$ is an
instruction specifying the output format via a delimiter such as
\texttt{\textbackslash boxed\{\}}. We define $\mathcal{F}(\mathcal{I})$ as the
format constraint induced by the instruction $\mathcal{I}$ (e.g., the regular
expression that the verifier matches against the response). Each supervised
task instance carries a ground-truth answer $a^*=a^*(\bm{x})$; semantics-preserving
transformations of $\bm{x}$ inherit the same answer label. When the answer
label is clear from context we suppress it in the reward notation. A verifier $V$
returns a binary reward $r \in \{0,1\}$ indicating whether $\bm{y}$ satisfies
the task. $R(\bm{y},\bm{x};a^*) \in [R_{\min}, R_{\max}]$ is defined as a more
general reward, with $R = V$ in the standard RLVR
setting. Let $\mathcal{D}$ denote the distribution over
task instances (visual-question pairs together with the fixed training-time
instruction and answer labels). The standard RLVR objective (e.g.,
GRPO~\citep{shao2024deepseekmath}) is
\begin{equation}
  \max_\theta\;
  \mathbb{E}_{\bm{x} \sim \mathcal{D}}\,
  \mathbb{E}_{\bm{y} \sim \pi_\theta(\cdot \mid \bm{x})}\!
  \bigl[V(\bm{y},\bm{x})\bigr].
  \label{eq:std_obj_main}
\end{equation}
GRPO optimizes Eq.~\eqref{eq:std_obj_main} via group-relative advantages on
$G$ rollouts per prompt; we define $A^{(i)}$ as the group-normalized
advantage of rollout $i$ (full form in \suppsec{E}).
For a finite set $S$, $\Delta(S)$ denotes the probability simplex on $S$; for
infinite $S$, read $\Delta(S)$ as probability measures on the relevant
measurable space.
Multi-template RLVR partially addresses prompt drift by drawing
$\bm{x}$ together with a template transformation
$\tau \sim P_{\text{train}}$, where
$P_{\text{train}}\in\Delta(\mathcal{T}_{\text{train}})$ is the training-time
distribution over the training transformation set
$\mathcal{T}_{\text{train}}$ of semantics-preserving instruction templates.
Together with
$\mathcal{D}$, this induces the joint training distribution.
In GRPO and its variants the verifier of Eq.~\eqref{eq:std_obj_main}
factorizes as a product of a format check and a correctness check,

\begin{equation}
  V(\bm{y},\bm{x})
  =
  V_f(\bm{y}, \mathcal{I})\,
  V_c(\bm{y}, \bm{q}, \bm{v}),
  \ \ \
  V_f,V_c\in\{0,1\},
  \label{eq:V_multiplicative}
\end{equation}
where $V_f = 1$ iff $\bm{y}$ matches the format constraint
$\mathcal{F}(\mathcal{I})$ and $V_c = 1$ iff $\bm{y}$ encodes the correct
answer. Three of the four outcomes $(V_f, V_c) \in \{0,1\}^2$ map to $V = 0$
and only $(1,1)$ maps to $V = 1$, so under group-relative advantage
normalization (\suppsec{E}) the three zero-reward
rollouts share the same scalar advantage; the reward/advantage signal does not
distinguish a format failure from a semantic failure or a joint failure. We identify two issues from this
observation: a reward-variance bias that acts as a proxy for the gradient signal
(Section~\ref{sec:theory_snr}, \suppthm{B.1}) and an objective-level
non-identifiability (Section~\ref{sec:theory_nonid}, \suppprop{B.5}).

\subsection{Reward Entanglement}
\label{sec:theory_snr}

With $p_f := \Pr[V_f{=}1]$ and $p_c := \Pr[V_c{=}1]$ over rollouts on a fixed
prompt $\bm{x}$, define the \emph{within-format reward-variance share}
$\Phi_{\mathrm{wf}} := \mathbb{E}[\mathrm{Var}(V\mid V_f)]/\mathrm{Var}(V)$, the fraction of
reward variance informative about $V_c$. Under an independent verifier
($V_f\perp V_c$) it equals $(1-p_c)/(1-p_f p_c)$, strictly increasing in $p_f$
with $\Phi_{\mathrm{wf}}\to1$ as $p_f\uparrow1$ (\suppthm{B.1}). Thus low format compliance leaves only the
limiting share $1-p_c$ of reward variance tied to answer correctness, while high
format compliance makes the remaining reward variance mostly semantic.
$\Phi_{\mathrm{wf}}$ is a reward-side proxy, not the full policy-gradient SNR; positive
$V_f$--$V_c$ correlation only lowers it when the marginal $p_c$ is held fixed
within the feasible conditional-rate range
(\suppsec{B.2}), and DTR removes the $p_f$-dependence
from its lower bound (\suppthm{B.8}).

\subsection{Prompt Dependency}
\label{sec:theory_nonid}

Let $\mathcal{T}$ be the family of semantics-preserving prompt transformations,
and for a task instance $\bm{x} = (\bm{v}, \bm{q}, \mathcal{I})$ let
$\llbracket \bm{x}\rrbracket := \{\tau(\bm{x}):\tau\in\mathcal{T}\}$ be the
\emph{semantic equivalence class}. In the fixed-$\bm{x}$ analysis below,
distributions such as $P_{\text{train}}$, $P_{\text{test}}$, and
$\rho_{\bm{x}}$ are distributions over transformations $\tau$; their
push-forward through $\tau\mapsto\tau(\bm{x})$ is the corresponding prompt
distribution on $\llbracket\bm{x}\rrbracket$. For a policy $\pi$ and
transformation distribution $\rho$, define
\begin{equation}
  \mathcal{R}(\pi,\rho;\bm{x})
  :=
  \mathbb{E}_{\tau\sim\rho}\,
  \mathbb{E}_{\bm{y}\sim\pi(\cdot\mid\tau(\bm{x}))}
  \bigl[R(\bm{y},\tau(\bm{x}))\bigr],
  \label{eq:reward_functional}
\end{equation}
and $\mathcal{R}(\pi,\tau;\bm{x})$ for the point mass $\rho=\delta_\tau$.
Consider restricted training coverage ($P_{\text{train}}$ supports only
$\mathcal{T}_{\text{train}}\subsetneq\mathcal{T}$) with a separated unseen format
$\tau_B$ whose verifier constraint is disjoint from the training formats. Under a
$\tau$-dependent verifier, two policies that agree on the training support and
coincide away from $\tau_B$---one
emitting a fixed training-time format on $\tau_B$, the other following the
requested format---are both $P_{\text{train}}$-optimal, yet differ by
$m_B\Delta_R$ on any test distribution placing mass $m_B$ on $\tau_B$
($\Delta_R := R_{\max}-R_{\min}$). The training objective therefore cannot
identify the prompt-invariant policy: some $P_{\text{train}}$-optimal policy has a
$P_{\text{test}}$ comparator gap $\ge m_B\Delta_R$
(\suppprop{B.5}, with assumptions and proof in \suppsec{B.3}).

\textbf{Extensions.} The single-point construction extends to a separated set
$\mathcal{B}\subseteq\mathcal{T}\setminus\mathcal{T}_{\text{train}}$ for an
integrated gap of $P_{\text{test}}(\mathcal{B})\Delta_R$, which is large only when
the separated set itself carries large test mass (an infinite $\mathcal{T}$ alone
does not suffice), and to additive dense rewards under an explicit
training-optimal condition (\suppsec{B.5}). A broader
measure-zero-support version, with a complexity tie-breaker for shortcuts, is
given in \suppsec{B.4}.

\subsection{Conditions for Prompt Invariance}
\label{sec:theory_dro}

We take the worst-case prompt distribution on $\llbracket \bm{x}\rrbracket$ as
the training target:
\begin{equation}
\begin{aligned}
  \max_\theta\quad
  \mathbb{E}_{\bm{x} \sim \mathcal{D}}
  \Bigg[
  \inf_{\rho_{\bm{x}} \in \Delta(\mathcal{T})}
  &\mathbb{E}_{\tau \sim \rho_{\bm{x}}}
  \mathbb{E}_{\bm{y} \sim \pi_\theta(\cdot \mid \tau(\bm{x}))}
  \bigl[R(\bm{y}, \tau(\bm{x}))\bigr]
  \Bigg].
\end{aligned}
  \label{eq:struct_dro}
\end{equation}
Eq.~\eqref{eq:struct_dro} is a group-DRO-style population objective~\citep{sagawa2020distributionally}
in which each transformation-induced prompt condition acts as a group within
the semantic equivalence class $\llbracket\bm{x}\rrbracket$.
The setup differs from classical group DRO in one practical respect: where
group DRO assumes a predefined partition of data, $\llbracket\bm{x}\rrbracket$
is \emph{operator-generated} (semantics-preserving paraphrases, format swaps,
NLI-filtered rewrites). $\mathcal{C}(\bm{x}) \subseteq \mathcal{T}$ denotes
a finite semantics-preserving transformation pool whose elements satisfy
$\tau(\bm{x})\in\llbracket\bm{x}\rrbracket$, used to state (C2).
Eq.~\eqref{eq:struct_dro} is a population reference for the ideal group-DRO
target. The adversary of Section~\ref{sec:method_adv} is a differentiable
embedding-space relaxation, not a certificate that arbitrary perturbations are
on-manifold elements of $\mathcal{T}$.

\textbf{Two design conditions} (full statements in
\suppsec{A}). \textbf{(C1) Reward Decomposition}
(\suppdef{A.1}): $R = R_t + \alpha_{\mathrm{rew}} R_s$ with fixed
$\alpha_{\mathrm{rew}}>0$,
$R_t \in \{-1, 0, +1\}$ format-aware, and a bounded semantic score
$R_s \in [0, 1]$ giving, for the stochastic rollout laws in the policy class
on which C1 is asserted,
$\mathrm{Var}(R\mid V_f=v') \ge c > 0$ whenever the format bucket
$\{V_f=v'\}$ has positive probability, so the
reward-variance share
$\Phi_{\mathrm{wf}}^R := \mathbb{E}[\mathrm{Var}(R\mid V_f)]/\mathrm{Var}(R)$ is
bounded below independently of $p_f$; boundedness alone is not the
non-degeneracy assumption. \textbf{(C2) Invariance Control}
(\suppdef{A.2}): under the prompt-conditioned answer marginal induced by
the deterministic extractor $\mathcal{A}_{\tau}$,
$\mathbb{E}_{\tau_1, \tau_2 \sim \rho_{\bm{x}}}
[D(\pi^{\mathcal{A}_{\tau_1}}_\theta(\cdot\mid\tau_1(\bm{x}))\,\|\,\pi^{\mathcal{A}_{\tau_2}}_\theta(\cdot\mid\tau_2(\bm{x})))]
\le \varepsilon$ for some non-degenerate coverage
$\rho_{\bm{x}} \in \Delta(\mathcal{C}(\bm{x}))$ and an $f$-divergence $D$
with Pinsker constant $\kappa_D$ (for the KL used here, $\kappa_D = 1/\sqrt 2$).
We require invariance only at the answer level; the per-token
adversarial-to-clean KL penalty
of Section~\ref{sec:method_cons} is a tractable training surrogate, not a strict
upper bound on (C2) (\suppsec{A}).
For the population-level comparison below, the delimiter extractor is assumed to be
format-consistent: if a response in $\mathcal{F}_\tau$ encodes answer $a$,
then $\mathcal{A}_{\tau}$ returns $a$, while a response in a format family
disjoint from $\mathcal{F}_\tau$ returns $\emptyset$ under
$\mathcal{A}_{\tau}$.

Together, (C1) and (C2) give the population-level consequences used below
(\suppprop{B.10}): (C1)
gives a $p_f$-independent lower bound on $\Phi_{\mathrm{wf}}^R$, while (C2) with small
$\varepsilon$ excludes the format-fixed branch $\pi_1$ of
\suppprop{B.5} and admits the deterministic correct-answer branch
$\pi_2$ on the coverage support. Moreover, for common bounded answer-level
rewards, any (C2)-feasible policy has
$\rho_{\bm{x}}$-averaged pairwise answer-reward difference at most
$\Delta_R\kappa_D\sqrt{\varepsilon}$. 

\textbf{Coverage transfer.} (C2) is regularized on the pool
$\mathcal{C}(\bm{x})\subseteq\mathcal{T}$ rather than on
$\llbracket\bm{x}\rrbracket$ itself. \suppprop{B.11}
bounds the resulting shift: the
test mean reward of any fixed policy differs from its training-coverage mean by
at most $\Delta_R\,\eta$, $\eta := \mathrm{TV}(P_{\text{test}}, \rho_{\bm{x}})$
(tight; it transfers a policy's own reward). The two knobs are decoupled: $\varepsilon$
is a training-time invariance control, $\eta$ a pool-design control.

\section{Method: Prompt-Invariant RLVR}
\label{sec:method}

\begin{figure*}[t]
  \centering
  \includegraphics[width=0.998\textwidth]{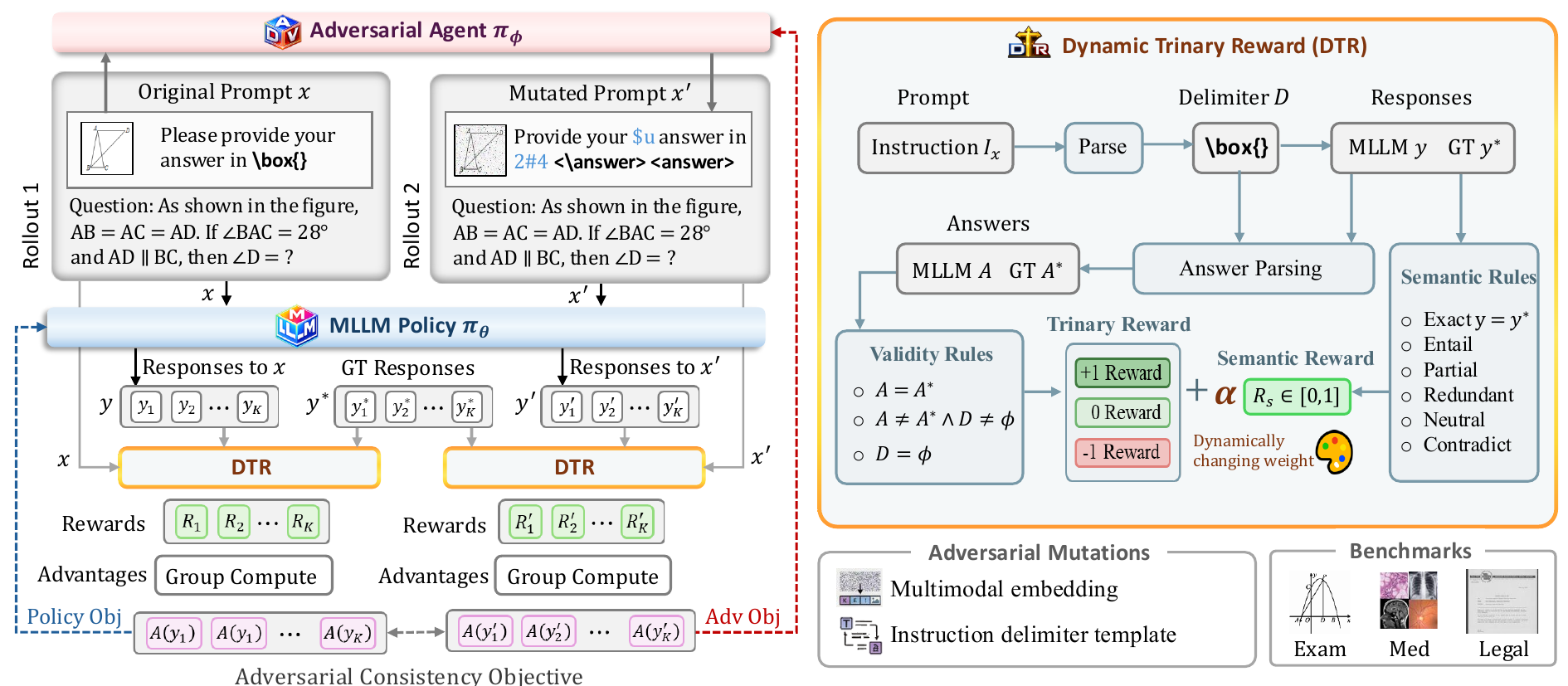}
  \vspace{-0.5em}
  \caption{PIRL targets reward decomposition via DTR + DAN and prompt-invariant
    learning via an embedding-space adversary plus a consistency regularizer.}
    \vspace{-1.5em}
  \label{fig:method}
\end{figure*}

PIRL instantiates (C1) through a Dynamic Trinary Reward (DTR,
Section~\ref{sec:method_dtr}), applied to the standard GRPO advantage with a small
normalization adjustment
(DAN, Section~\ref{sec:method_dan}); it regularizes toward (C2) through an
embedding-space adversary plus a consistency regularizer
(Section~\ref{sec:method_adv}). The reward-adversarial
sampler (Section~\ref{sec:method_objective}) is a hard-prompt sampling heuristic
inspired by Eq.~\eqref{eq:struct_dro}; the policy update remains a clipped GRPO
surrogate with DAN advantages. In the full configuration, we first draw a
discrete training template $\tau_0\sim P_{\mathrm{train}}$ from the same
five-template pool used by GRPO-MT and then apply the embedding adversary around
$\tau_0(\bm{x})$. We absorb this initial template draw into the prompt notation
below; the no-MT ablation sets $\tau_0$ to the identity.

\subsection{Dynamic Trinary Reward (DTR)}
\label{sec:method_dtr}

DTR rewards each rollout with $R = R_t + \alpha_{\mathrm{rew}} R_s$, where
$R_t \in \{-1, 0, +1\}$ is a format-aware trinary verifier and $R_s \in [0,1]$
is a format-independent semantic score. Let
$a_\tau(\bm{y}) := \mathcal{A}(\bm{y}; d_\tau)$ be the answer extracted at the
delimiter $d_\tau$ specified by the current instruction, and let $a^*$ be the
ground-truth answer. The trinary verifier is
\begin{equation}
  R_t = \begin{cases}
    +1, & a_\tau(\bm{y}) = a^*, \\
    \phantom{+}0, & a_\tau(\bm{y}) \notin \{\emptyset, a^*\}, \\
    -1, & a_\tau(\bm{y}) = \emptyset,
  \end{cases}
  \label{eq:trinary}
\end{equation}
which distinguishes correct answer ($+1$), wrong answer in correct format
($0$), and format failure ($-1$). $R_s$ is the entailment score from a fixed NLI
scorer~\citep{camburu2018snli}; for the variance-floor theorem, $R_s$ is not
collapsed to a single $\emptyset$ value on format failures. The total DTR reward is
\begin{equation}
  R_{\mathrm{DTR}}(\bm{y},\bm{x};a^*) = R_t + \alpha_{\mathrm{rew}} R_s,
  \label{eq:dtr}
\end{equation}
with a fixed reward weight $\alpha_{\mathrm{rew}} > 0$. In our experiments
$\alpha_{\mathrm{rew}}=0.01<1$, which preserves the ordering
correct-and-formatted $>$ wrong-but-formatted $>$ format failure for all
$R_s\in[0,1]$; the variance result itself only needs
$\alpha_{\mathrm{rew}}>0$. For a transformed prompt $\tau(\bm{x})$, the
ground-truth label remains the label of the underlying task instance; when clear
we write $R_{\mathrm{DTR}}(\bm{y},\tau(\bm{x}))$ as shorthand for
$R_{\mathrm{DTR}}(\bm{y},\tau(\bm{x});a^*)$.

Under the non-degeneracy condition formalized in
\suppsec{B.6}, the bounded semantic score $R_s$
injects within-bucket reward variance even when the format outcome is fixed, so
the reward-variance lower bound does not collapse as the format-valid rate
$p_f$ changes: the $p_f$-dependence of \suppthm{B.1} drops out, and
while the realized $\Phi_{\mathrm{wf}}^{\mathrm{DTR}}$ can still vary with $p_f$ through
$\mathrm{Var}(R)$, its numerator has a lower bound independent of $p_f$. We defer the
formal statement, its non-degeneracy assumption, and the proof to
\suppsec{B.6}.

\paragraph{Decoupled advantage normalization (DAN).}\label{sec:method_dan}
Naive group-normalization $\hat A = \mathrm{Norm}(R_t + \alpha_{\mathrm{rew}} R_s)$
couples $R_t$ and $R_s$ at the batch level when their empirical scales
differ. We z-score $R_t$ and $R_s$ separately within each rollout group,
combine, and re-normalize at the batch level:
\begin{equation}
\begin{aligned}
  Z_i
  &=
  \mathrm{Norm}_{\mathcal{G}}(R_t^{(i)})
  + \alpha_{\mathrm{adv}}\,\mathrm{Norm}_{\mathcal{G}}(R_s^{(i)}),
  \\
  \hat A_i
  &=
  \mathrm{clip}\!\left(\mathrm{Norm}_{\mathcal{B}}(Z_i),\;-5,5\right),
  \label{eq:gdan_adv}
\end{aligned}
\end{equation}
where $\mathrm{Norm}_S(z) = (z-\bar z)/(\sigma_z + \epsilon_{\mathrm{norm}})$
uses the same numerical-stability $\epsilon_{\mathrm{norm}}>0$ as standard
GRPO. Here $\alpha_{\mathrm{adv}}>0$ is the advantage-mixing coefficient after
z-scoring; it is distinct from the scalar reward coefficient
$\alpha_{\mathrm{rew}}$ in Eq.~\eqref{eq:dtr}. DAN does not change the DTR
reward definition used in \suppthm{B.8}; in finite batches
the realized advantage scale depends on the within-group correlation between
the two components.

\subsection{Adversarial agent and consistency}
\label{sec:method_adv}

We approximate the worst-case distribution in Eq.~\eqref{eq:struct_dro} with an
adversarial sampler $q_\phi$ that perturbs the prompt in embedding space. In
the projected-gradient implementation, $q_\phi$ can be read as the point mass
at the final perturbation; a stochastic perturbation head gives the distributional
version of the same notation. This sampler is distinct from the C2 coverage
distribution $\rho_{\bm{x}}$; the consistency term regularizes sampled
clean/adversarial pairs rather than certifying the full C2 expectation.
Given an instance $\bm{x}=(\bm{v},\bm{q},\mathcal{I})$ and the current policy,
let $\bm{e}_{\mathcal{I}}(\bm{x})$ denote the instruction-token embeddings.
The main adversary emits a bounded additive perturbation on this instruction
block,
$\tilde{\bm{e}}_{\mathcal{I}}(\bm{x})
= \bm{e}_{\mathcal{I}}(\bm{x}) + \bm{\delta}$ with
$\|\bm{\delta}\|_2 \le \epsilon_{\mathrm{emb}}$, and we write
$\tau_\phi(\bm{x})$ for the resulting perturbed prompt. The perturbation
explores the $\epsilon_{\mathrm{emb}}$-ball around the instruction
representation to drive down the policy's expected DTR reward, approximating
local prompt variation under reworded instructions; the image, question content,
and answer target are held fixed. For embedding-space perturbations, the surface
delimiter used by the extractor is not changed: $d_{\tau_\phi}=d_{\bm{x}}$.
Discrete template transformations use their own requested delimiter $d_\tau$.

\paragraph{Consistency regularization.}\label{sec:method_cons}
The adversary alone only discovers low-reward local prompt perturbations. We add an
adversarial-to-clean KL penalty on paired token distributions. For a prefix distribution
$\nu_t$ (sampled-prefix for the sequence-KL view, teacher-forced in the
implementation) over $\bm{y}_{<t}$, let
$p_t^{\bm{x}}=\pi_\theta(\cdot\mid \bm{y}_{<t},\bm{x})$ and
$p_t^{\tau}=\pi_\theta(\cdot\mid \bm{y}_{<t},\tau_\phi(\bm{x}))$. We use
\begin{equation}
\begin{aligned}
  \mathcal{L}_{\mathrm{cons}}(\theta;\bm{x},\phi)
  &=
  \frac{1}{T}\sum_{t=1}^{T}
  \mathbb{E}_{\bm{y}_{<t}\sim\nu_t} D_{\mathrm{KL}}\!\left(
    p_t^{\tau}\,\middle\|\,\operatorname{sg}(p_t^{\bm{x}})
  \right),
  \label{eq:cons}
\end{aligned}
\end{equation}
where $\operatorname{sg}(\cdot)$ stops gradients through the clean-prompt anchor. This
is a tractable per-token proxy for the sequence-level divergence; we adopt
the per-token form because it is differentiable end-to-end and its per-token
gradients are dense (its relation to the full-sequence $D_{\mathrm{KL}}$ is
discussed in \suppsec{A}).

\begin{table*}[t]
  \centering
  \begin{threeparttable}
    \caption{Accuracy (\%) under standard and template-stress (T-Stress) evaluation. MMK12 pools the Math/Science subsets; Olym-Phys is OlympiadBench-Physics; Legal VQA is DocVQA. ID/OOD: in-/out-of-distribution; MT: Multi-Template. Best per model category in \textbf{bold}.}
    \vspace{-1em}
    \label{tab:main_results_comprehensive}
    \fontsize{8}{10}\selectfont
    \setlength{\tabcolsep}{4pt}
    \renewcommand{\arraystretch}{0.5}

    \begin{tabular}{@{}
        l  
        l  
        >{\centering\arraybackslash}p{1.65cm}  
        >{\centering\arraybackslash}p{1.65cm}  
        >{\centering\arraybackslash}p{1.65cm}  
        |
        >{\centering\arraybackslash}p{1.35cm}  
        >{\centering\arraybackslash}p{1.35cm}  
        >{\centering\arraybackslash}p{1.35cm}  
        |
        >{\centering\arraybackslash}p{1.5cm}  
      @{}}
      \toprule
      \multirow{3}{*}{\textbf{Backbone}}
      & \multirow{3}{*}{\textbf{Method}}
      & \multicolumn{3}{c|}{\textbf{Exam VQA}}
      & \multicolumn{3}{c|}{\textbf{Medical VQA}}
      & \multicolumn{1}{c}{\textbf{Legal VQA}} \\
      \cmidrule(lr){3-5} \cmidrule(lr){6-8} \cmidrule(lr){9-9}
      & & \multicolumn{2}{c}{\textbf{ID}} & \multicolumn{1}{c|}{\textbf{OOD}}
      & \multicolumn{2}{c}{\textbf{ID}} & \multicolumn{1}{c|}{\textbf{OOD}}
      & \multicolumn{1}{c}{\textbf{OOD}} \\
      \cmidrule(lr){3-4} \cmidrule(lr){5-5} \cmidrule(lr){6-7} \cmidrule(lr){8-8} \cmidrule(lr){9-9}
      & & \fontsize{7.5}{11}\selectfont\textbf{MMK12}
      & \fontsize{7.5}{11}\selectfont\textbf{MathVista}
      & \fontsize{7.5}{11}\selectfont\textbf{Olym-Phys}
      & \fontsize{7.5}{11}\selectfont\textbf{RAD}
      & \fontsize{7.5}{11}\selectfont\textbf{Path}
      & \fontsize{7.5}{11}\selectfont\textbf{GMAI}
      & \fontsize{7.5}{11}\selectfont\textbf{DocVQA} \\
      \midrule

      \multirow{12}{*}{\shortstack[c]{\textbf{Qwen2.5}\\\textbf{VL-7B}}} & Base Model
      & 53.1 & 68.4 & 8.5 & 64.7 & 43.0 & \textbf{44.8} & 66.8 \\
      & \child{T-Stress}
      & 54.1 & \textbf{69.6} & 7.2 & 66.3 & 42.4 & \textbf{42.4} & 67.2 \\
      & \child{$\Delta \mathrm{T}$}
      & +1.0${\scriptscriptstyle\pm0.4}$
      & +1.2${\scriptscriptstyle\pm0.4}$
      & -1.3${\scriptscriptstyle\pm0.7}$
      & +1.6${\scriptscriptstyle\pm0.5}$
      & -0.6${\scriptscriptstyle\pm0.6}$
      & -2.4${\scriptscriptstyle\pm0.6}$
      & +0.4${\scriptscriptstyle\pm0.5}$ \\
      \cmidrule(lr){2-9}
      & GRPO
      & 54.9 & \textbf{70.1} & 7.5 & \textbf{73.5} & 48.1 & 42.1 & 69.1 \\
      & \child{T-Stress}
      & 54.1 & 65.7 & 5.9 & 71.2 & 41.4 & 39.7 & 66.7 \\
      & \child{$\Delta \mathrm{T}$}
      & -0.8${\scriptscriptstyle\pm0.4}$
      & -4.4${\scriptscriptstyle\pm0.4}$
      & -1.6${\scriptscriptstyle\pm0.7}$
      & -2.3${\scriptscriptstyle\pm0.5}$
      & -6.7${\scriptscriptstyle\pm0.6}$
      & -2.4${\scriptscriptstyle\pm0.6}$
      & -2.4${\scriptscriptstyle\pm0.5}$ \\
      \cmidrule(lr){2-9}
      & GRPO-MT
      & 54.9 & 70.0 & 7.8 & 71.0 & 45.8 & 42.2 & 68.9 \\
      & \child{T-Stress}
      & 54.0 & 69.3 & 7.6 & 69.7 & 43.2 & 40.9 & 68.0 \\
      & \child{$\Delta \mathrm{T}$}
      & -0.9${\scriptscriptstyle\pm0.4}$
      & -0.7${\scriptscriptstyle\pm0.4}$
      & -0.2${\scriptscriptstyle\pm0.7}$
      & -1.3${\scriptscriptstyle\pm0.5}$
      & -2.6${\scriptscriptstyle\pm0.6}$
      & -1.3${\scriptscriptstyle\pm0.6}$
      & -0.9${\scriptscriptstyle\pm0.5}$ \\
      \cmidrule(lr){2-9}
      & \textbf{PIRL (Ours)}
      & \textbf{55.3} & 68.6 & \textbf{9.6} & 72.6 & \textbf{49.3} & \textbf{44.8} & \textbf{69.5} \\
      & \child{T-Stress}
      & \textbf{55.7} & 68.9 & \textbf{10.7} & \textbf{73.4} & \textbf{48.5} & 42.1 & \textbf{69.7} \\
      & \child{$\Delta \mathrm{T}$}
      & +0.4${\scriptscriptstyle\pm0.4}$
      & +0.3${\scriptscriptstyle\pm0.4}$
      & +1.1${\scriptscriptstyle\pm0.7}$
      & +0.8${\scriptscriptstyle\pm0.5}$
      & -0.8${\scriptscriptstyle\pm0.6}$
      & -2.7${\scriptscriptstyle\pm0.6}$
      & +0.2${\scriptscriptstyle\pm0.5}$ \\
      \midrule

      \multirow{12}{*}{\shortstack[c]{\textbf{Qwen3}\\\textbf{VL-8B}}} & Base Model
      & 60.3 & 68.9 & 25.3 & 68.1 & 49.1 & 46.0 & 70.3 \\
      & \child{T-Stress}
      & 59.9 & 69.6 & 25.7 & 68.7 & \textbf{49.2} & \textbf{46.2} & 71.0 \\
      & \child{$\Delta \mathrm{T}$}
      & -0.4${\scriptscriptstyle\pm0.4}$
      & +0.7${\scriptscriptstyle\pm0.4}$
      & +0.4${\scriptscriptstyle\pm0.7}$
      & +0.6${\scriptscriptstyle\pm0.5}$
      & +0.1${\scriptscriptstyle\pm0.6}$
      & +0.2${\scriptscriptstyle\pm0.6}$
      & +0.7${\scriptscriptstyle\pm0.5}$ \\
      \cmidrule(lr){2-9}
      & GRPO
      & 60.5 & 69.1 & 27.1 & 77.2 & 49.9 & 46.2 & 72.7 \\
      & \child{T-Stress}
      & 60.0 & 67.5 & 26.0 & 76.1 & 48.1 & 44.8 & 71.4 \\
      & \child{$\Delta \mathrm{T}$}
      & -0.5${\scriptscriptstyle\pm0.4}$
      & -1.6${\scriptscriptstyle\pm0.4}$
      & -1.1${\scriptscriptstyle\pm0.7}$
      & -1.1${\scriptscriptstyle\pm0.5}$
      & -1.8${\scriptscriptstyle\pm0.6}$
      & -1.4${\scriptscriptstyle\pm0.6}$
      & -1.3${\scriptscriptstyle\pm0.5}$ \\
      \cmidrule(lr){2-9}
      & GRPO-MT
      & 58.7 & \textbf{72.8} & 21.5 & 77.4 & 49.8 & 46.2 & 72.6 \\
      & \child{T-Stress}
      & 57.8 & \textbf{73.1} & 21.5 & 76.1 & 47.2 & 44.9 & 72.2 \\
      & \child{$\Delta \mathrm{T}$}
      & -0.9${\scriptscriptstyle\pm0.4}$
      & +0.3${\scriptscriptstyle\pm0.4}$
      & 0.0${\scriptscriptstyle\pm0.7}$
      & -1.3${\scriptscriptstyle\pm0.5}$
      & -2.6${\scriptscriptstyle\pm0.6}$
      & -1.3${\scriptscriptstyle\pm0.6}$
      & -0.4${\scriptscriptstyle\pm0.5}$ \\
      \cmidrule(lr){2-9}
      & \textbf{PIRL (Ours)}
      & \textbf{60.9} & 72.0 & \textbf{27.4} & \textbf{78.8} & \textbf{50.0} & \textbf{46.3} & \textbf{73.1} \\
      & \child{T-Stress}
      & \textbf{60.1} & 72.4 & \textbf{26.3} & \textbf{79.6} & \textbf{49.2} & 43.6 & \textbf{73.6} \\
      & \child{$\Delta \mathrm{T}$}
      & -0.8${\scriptscriptstyle\pm0.4}$
      & +0.4${\scriptscriptstyle\pm0.4}$
      & -1.1${\scriptscriptstyle\pm0.7}$
      & +0.8${\scriptscriptstyle\pm0.5}$
      & -0.8${\scriptscriptstyle\pm0.6}$
      & -2.7${\scriptscriptstyle\pm0.6}$
      & +0.5${\scriptscriptstyle\pm0.5}$ \\
      \bottomrule
    \end{tabular}
    \begin{tablenotes}[para,flushleft]
      \fontsize{7.8}{8}\selectfont
      \vspace{0.5ex}
    \end{tablenotes}
  \end{threeparttable}
  \vspace{-1.5em}
\end{table*}

\begin{table*}[t]
\centering
\begin{threeparttable}
\caption{Accuracy (\%) under standard and dynamic-test (D-Test) evaluation over the full operator pool. For each backbone, PIRL has the smallest mean gap across the seven benchmarks among trained methods.
Best per model category in \textbf{bold}.}
\vspace{-1.5em}
\label{tab:main_results_comprehensive_dynamic}
\fontsize{8}{10}\selectfont
\setlength{\tabcolsep}{4pt}
\renewcommand{\arraystretch}{0.7}
\begin{tabular}{@{}
l  
l  
>{\centering\arraybackslash}p{1.65cm}
>{\centering\arraybackslash}p{1.65cm}
>{\centering\arraybackslash}p{1.65cm}
|
>{\centering\arraybackslash}p{1.35cm}
>{\centering\arraybackslash}p{1.35cm}
>{\centering\arraybackslash}p{1.35cm}
|
>{\centering\arraybackslash}p{1.5cm}
@{}}
\toprule
      \multirow{3}{*}{\textbf{Backbone}}
      & \multirow{3}{*}{\textbf{Method}}
      & \multicolumn{3}{c|}{\textbf{Exam VQA}}
      & \multicolumn{3}{c|}{\textbf{Medical VQA}}
      & \multicolumn{1}{c}{\textbf{Legal VQA}} \\
      \cmidrule(lr){3-5} \cmidrule(lr){6-8} \cmidrule(lr){9-9}
      & & \multicolumn{2}{c}{\textbf{ID}} & \multicolumn{1}{c|}{\textbf{OOD}}
      & \multicolumn{2}{c}{\textbf{ID}} & \multicolumn{1}{c|}{\textbf{OOD}}
      & \multicolumn{1}{c}{\textbf{OOD}} \\
      \cmidrule(lr){3-4} \cmidrule(lr){5-5} \cmidrule(lr){6-7} \cmidrule(lr){8-8} \cmidrule(lr){9-9}
      & & \fontsize{7.5}{11}\selectfont\textbf{MMK12}
      & \fontsize{7.5}{11}\selectfont\textbf{MathVista}
      & \fontsize{7.5}{11}\selectfont\textbf{Olym-Phys}
      & \fontsize{7.5}{11}\selectfont\textbf{RAD}
      & \fontsize{7.5}{11}\selectfont\textbf{Path}
      & \fontsize{7.5}{11}\selectfont\textbf{GMAI}
      & \fontsize{7.5}{11}\selectfont\textbf{DocVQA} \\
      \midrule

\multirow{12}{*}{\shortstack[c]{\textbf{Qwen2.5}\\\textbf{VL-7B}}} & Base Model
& 53.1 & 68.4 & 8.5 & 64.7 & 43.0 & \textbf{44.8} & 66.8 \\
& \child{D-Test}
& 53.5 & 21.4 & 4.4 & \textbf{66.1} & 40.4 & 37.3 & 46.1 \\
& \child{$\Delta \mathrm{D}$}
      & +0.4${\scriptscriptstyle\pm0.5}$
      & -47.0${\scriptscriptstyle\pm0.8}$
      & -4.1${\scriptscriptstyle\pm0.7}$
      & +1.4${\scriptscriptstyle\pm0.6}$
      & -2.6${\scriptscriptstyle\pm0.6}$
      & -7.5${\scriptscriptstyle\pm0.6}$
      & -20.7${\scriptscriptstyle\pm0.5}$ \\
\cmidrule(lr){2-9}
& GRPO
& 54.9 & \textbf{70.1} & 7.5 & \textbf{73.5} & 48.1 & 42.1 & 69.1 \\
& \child{D-Test}
& 51.3 & \textbf{21.6} & 4.4 & 61.4 & 42.3 & 34.5 & 46.3 \\
& \child{$\Delta \mathrm{D}$}
      & -3.6${\scriptscriptstyle\pm0.5}$
      & -48.5${\scriptscriptstyle\pm0.8}$
      & -3.1${\scriptscriptstyle\pm0.7}$
      & -12.1${\scriptscriptstyle\pm0.6}$
      & -5.8${\scriptscriptstyle\pm0.6}$
      & -7.6${\scriptscriptstyle\pm0.6}$
      & -22.8${\scriptscriptstyle\pm0.5}$ \\
\cmidrule(lr){2-9}
& GRPO-MT
& 54.9 & 70.0 & 7.8 & 71.0 & 45.8 & 42.2 & 68.9 \\
& \child{D-Test}
& 54.2 & 21.3 & 4.3 & 61.9 & 42.1 & 34.5 & 47.7 \\
& \child{$\Delta \mathrm{D}$}
      & -0.7${\scriptscriptstyle\pm0.5}$
      & -48.7${\scriptscriptstyle\pm0.8}$
      & -3.5${\scriptscriptstyle\pm0.7}$
      & -9.1${\scriptscriptstyle\pm0.6}$
      & -3.7${\scriptscriptstyle\pm0.6}$
      & -7.7${\scriptscriptstyle\pm0.6}$
      & -21.2${\scriptscriptstyle\pm0.5}$ \\
\cmidrule(lr){2-9}
& \textbf{PIRL (Ours)}
& \textbf{55.3} & 68.6 & \textbf{9.6} & 72.6 & \textbf{49.3} & \textbf{44.8} & \textbf{69.5} \\
& \child{D-Test}
& \textbf{54.2} & 21.4 & \textbf{5.5} & 65.3 & \textbf{46.2} & \textbf{37.6} & \textbf{49.0} \\
& \child{$\Delta \mathrm{D}$}
      & -1.1${\scriptscriptstyle\pm0.5}$
      & -47.2${\scriptscriptstyle\pm0.8}$
      & -4.1${\scriptscriptstyle\pm0.7}$
      & -7.3${\scriptscriptstyle\pm0.6}$
      & -3.1${\scriptscriptstyle\pm0.6}$
      & -7.2${\scriptscriptstyle\pm0.6}$
      & -20.5${\scriptscriptstyle\pm0.5}$ \\
\midrule

\multirow{12}{*}{\shortstack[c]{\textbf{Qwen3}\\\textbf{VL-8B}}} & Base Model
& 60.3 & 68.9 & 25.3 & 68.1 & 49.1 & 46.0 & 70.3 \\
& \child{D-Test}
& \textbf{60.2} & 21.8 & 20.7 & 69.0 & 46.7 & 38.6 & 49.8 \\
& \child{$\Delta \mathrm{D}$}
      & -0.1${\scriptscriptstyle\pm0.5}$
      & -47.1${\scriptscriptstyle\pm0.8}$
      & -4.6${\scriptscriptstyle\pm0.7}$
      & +0.9${\scriptscriptstyle\pm0.6}$
      & -2.4${\scriptscriptstyle\pm0.6}$
      & -7.4${\scriptscriptstyle\pm0.6}$
      & -20.5${\scriptscriptstyle\pm0.5}$ \\
\cmidrule(lr){2-9}
& GRPO
& 60.5 & 69.1 & 27.1 & 77.2 & 49.9 & 46.2 & 72.7 \\
& \child{D-Test}
& 56.5 & 20.8 & \textbf{23.7} & 64.9 & 44.1 & 38.9 & 50.7 \\
& \child{$\Delta \mathrm{D}$}
      & -4.0${\scriptscriptstyle\pm0.5}$
      & -48.3${\scriptscriptstyle\pm0.8}$
      & -3.4${\scriptscriptstyle\pm0.7}$
      & -12.3${\scriptscriptstyle\pm0.6}$
      & -5.8${\scriptscriptstyle\pm0.6}$
      & -7.3${\scriptscriptstyle\pm0.6}$
      & -22.0${\scriptscriptstyle\pm0.5}$ \\
\cmidrule(lr){2-9}
& GRPO-MT
& 58.7 & \textbf{72.8} & 21.5 & 77.4 & 49.8 & 46.2 & 72.6 \\
& \child{D-Test}
& 58.4 & 23.8 & 18.2 & 68.2 & 46.2 & 38.6 & 51.8 \\
& \child{$\Delta \mathrm{D}$}
      & -0.3${\scriptscriptstyle\pm0.5}$
      & -49.0${\scriptscriptstyle\pm0.8}$
      & -3.3${\scriptscriptstyle\pm0.7}$
      & -9.2${\scriptscriptstyle\pm0.6}$
      & -3.6${\scriptscriptstyle\pm0.6}$
      & -7.6${\scriptscriptstyle\pm0.6}$
      & -20.8${\scriptscriptstyle\pm0.5}$ \\
\cmidrule(lr){2-9}
& \textbf{PIRL (Ours)}
& \textbf{60.9} & 72.0 & \textbf{27.4} & \textbf{78.8} & \textbf{50.0} & \textbf{46.3} & \textbf{73.1} \\
& \child{D-Test}
& 60.1 & \textbf{24.6} & 22.8 & \textbf{71.2} & \textbf{49.0} & \textbf{39.5} & \textbf{52.8} \\
& \child{$\Delta \mathrm{D}$}
      & -0.8${\scriptscriptstyle\pm0.5}$
      & -47.4${\scriptscriptstyle\pm0.8}$
      & -4.6${\scriptscriptstyle\pm0.7}$
      & -7.6${\scriptscriptstyle\pm0.6}$
      & -1.0${\scriptscriptstyle\pm0.6}$
      & -6.8${\scriptscriptstyle\pm0.6}$
      & -20.3${\scriptscriptstyle\pm0.5}$ \\
\bottomrule
\end{tabular}
\begin{tablenotes}[para,flushleft]
\fontsize{7.8}{8}\selectfont
\vspace{0.5ex}
\end{tablenotes}
\end{threeparttable}
\vspace{-1em}
\end{table*}

\begin{table*}[t]
  \centering
	  \caption{\textbf{Ablation} of multi-template (MT), dynamic trinary reward (DTR), and adversarial prompting (ADV). Mean$\pm$std over 3 seeds ($\Delta\mathrm{T}$ within $\pm 1.5$ pp inconclusive). Exam VQA Avg = mean of MMK12/MathVista/Olym-Phys.}
  \label{tab:ablation_final_filled}
  \resizebox{0.8\textwidth}{!}{
    \setlength{\tabcolsep}{3.2pt}
    \renewcommand{\arraystretch}{1.1}
    \begin{tabular}{ccc | ccccc | ccccc}
      \toprule
      \multicolumn{3}{c|}{\textbf{Components}} &
      \multicolumn{5}{c|}{\textbf{Exam VQA Avg.}} &
      \multicolumn{5}{c}{\textbf{Olym-Phys}} \\
      \cmidrule(lr){1-3}
      \cmidrule(lr){4-8}
      \cmidrule(lr){9-13}
      \textbf{MT} & \textbf{ADV} & \textbf{DTR} &
      \textbf{Std.} & \textbf{T-Stress} & \textbf{D-Test} & \textbf{$\Delta$T} & \textbf{$\Delta$D} &
      \textbf{Std.} & \textbf{T-Stress} & \textbf{D-Test} & \textbf{$\Delta$T} & \textbf{$\Delta$D} \\
      \midrule
      $\times$ & $\times$ & $\times$ &
      $44.2{\scriptscriptstyle\pm0.4}$ & $41.9{\scriptscriptstyle\pm0.5}$ & $25.8{\scriptscriptstyle\pm0.6}$ & $-2.3{\scriptscriptstyle\pm0.5}$ & $-18.4{\scriptscriptstyle\pm0.7}$ &
      $7.5{\scriptscriptstyle\pm0.5}$ & $5.9{\scriptscriptstyle\pm0.5}$ & $4.4{\scriptscriptstyle\pm0.5}$ & $-1.6{\scriptscriptstyle\pm0.5}$ & $-3.1{\scriptscriptstyle\pm0.5}$ \\
      \midrule

      $\times$ & $\checkmark$ & $\checkmark$ &
      $44.3{\scriptscriptstyle\pm0.4}$ & $43.4{\scriptscriptstyle\pm0.5}$ & $26.0{\scriptscriptstyle\pm0.7}$ & $-0.9{\scriptscriptstyle\pm0.5}$ & $-18.3{\scriptscriptstyle\pm0.7}$ &
      $10.3{\scriptscriptstyle\pm0.5}$ & $10.4{\scriptscriptstyle\pm0.5}$ & $4.4{\scriptscriptstyle\pm0.5}$ & $+0.1{\scriptscriptstyle\pm0.5}$ & $-5.9{\scriptscriptstyle\pm0.5}$ \\
      $\checkmark$ & $\times$ & $\checkmark$ &
      $44.4{\scriptscriptstyle\pm0.4}$ & $44.4{\scriptscriptstyle\pm0.5}$ & $26.8{\scriptscriptstyle\pm0.7}$ & 0.0${\scriptscriptstyle\pm0.5}$ & $-17.6{\scriptscriptstyle\pm0.7}$ &
      $10.1{\scriptscriptstyle\pm0.5}$ & $9.8{\scriptscriptstyle\pm0.5}$ & $4.2{\scriptscriptstyle\pm0.5}$ & $-0.3{\scriptscriptstyle\pm0.5}$ & $-5.9{\scriptscriptstyle\pm0.5}$ \\
      \midrule
      
      $\checkmark$ & $\checkmark$ & $\checkmark$ &
      $\textbf{44.5}{\scriptscriptstyle\pm0.4}$ & $\textbf{45.1}{\scriptscriptstyle\pm0.5}$ & $\textbf{27.0}{\scriptscriptstyle\pm0.7}$ & \textbf{$+0.6{\scriptscriptstyle\pm0.5}$} & $-17.5{\scriptscriptstyle\pm0.7}$ &
      $\textbf{9.6}{\scriptscriptstyle\pm0.5}$ & $\textbf{10.7}{\scriptscriptstyle\pm0.5}$ & $\textbf{5.5}{\scriptscriptstyle\pm0.5}$ & \textbf{$+1.1{\scriptscriptstyle\pm0.5}$} & $-4.1{\scriptscriptstyle\pm0.5}$ \\
      \bottomrule
  \end{tabular}}
  \vspace{-1.5em}
\end{table*}

\subsection{Overall objective}
\label{sec:method_objective}

Let $\mathcal{L}_{\mathrm{G}}^{\mathrm{DAN}}(\pi_\theta, q; R_t,R_s)$ denote
the clipped GRPO loss evaluated on prompts from $q$, with DAN advantages from
Eq.~\eqref{eq:gdan_adv} and an optional reference-KL penalty (full form in
\suppsec{E}); its coefficient is set to zero in our experiments. The ideal
relaxed adversarial target is
\begin{equation}
\begin{aligned}
  &\max_\phi\;
  \mathbb{E}_{\bm{x} \sim \mathcal{D}}\,
  \mathbb{E}_{\tau\sim q_\phi(\cdot\mid\bm{x})}
  \mathbb{E}_{\bm{y}\sim\pi_\theta(\cdot\mid\tau(\bm{x}))}
  \bigl[-R_{\mathrm{DTR}}(\bm{y},\tau(\bm{x}))\bigr],
  \label{eq:adv_update_main}
\end{aligned}
\end{equation}
with $q_\phi$ supported on the $\epsilon_{\mathrm{emb}}$ instruction-embedding
ball. In the algorithm, let $\widehat{\phi}_K(\theta)$ denote the parameter
returned by $K$ projected-ascent steps toward Eq.~\eqref{eq:adv_update_main},
and let $\bar\phi=\mathrm{sg}(\widehat{\phi}_K(\theta))$ denote the frozen
adversarial sample used by the outer update. The policy objective is
\begin{equation}
\begin{aligned}
  \min_\theta\;
  \mathbb{E}_{\bm{x} \sim \mathcal{D}}\!\Bigl[
    &\mathcal{L}_{\mathrm{G}}^{\mathrm{DAN}}\bigl(
      \pi_\theta,q_{\bar\phi}(\cdot\mid\bm{x});R_t,R_s
    \bigr) + \lambda\,\mathcal{L}_{\mathrm{cons}}(\theta;\bm{x},\bar\phi)
  \Bigr],
  \label{eq:policy_obj}
\end{aligned}
\end{equation}
During the policy update, the resulting perturbations are treated as fixed
adversarial samples; we do not differentiate through the inner ascent trajectory
or claim exact bilevel optimization. Advantages are normalized by DAN
(Eq.~\eqref{eq:gdan_adv}), and $\lambda > 0$ weights the invariance penalty.
\suppalg{2} gives the full loop.

\begin{figure}[t]
  \centering
  \includegraphics[width=0.7\linewidth]{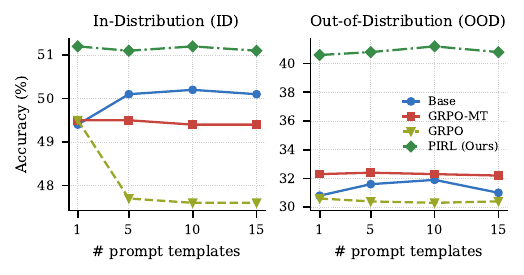}
  \vspace{-0.5em}
  \caption{\textbf{Effect of prompt-template count} on Exam VQA ID/OOD
  accuracy. GRPO degrades on ID as more unseen templates appear; GRPO-MT and
  PIRL stay flat, with PIRL highest on both.}
  \vspace{-0.5em}
  \label{fig:templates}
\end{figure}

\section{Experiments}
\label{sec:experiments}

\subsection{Setup}

We conduct extensive experiments to evaluate our framework across diverse benchmarks:
\textbf{Exam VQA} (train/ID eval
MMK12~\citep{meng2025mm}, ID eval MathVista~\citep{lu2024mathvista},
OOD eval Olym-Phys~\citep{he2024olympiadbench}), \textbf{Medical VQA}
(train Lingshu~\citep{xu2025lingshu}; ID eval VQA-RAD~\citep{lau2018dataset}
and PathVQA~\citep{he2020pathvqa}; OOD eval
GMAI-MMBench~\citep{chen2024gmaimmbench}), and \textbf{Legal VQA}
(train DUDE~\citep{vanlandeghem2023dude}; OOD eval
DocVQA~\citep{mathew2021docvqa}).
Base models: Qwen2.5-VL-7B~\citep{bai2025qwen25vl}
and Qwen3-VL-8B~\citep{bai2025qwen3vl}. Baselines:
GRPO~\citep{shao2024deepseekmath} and GRPO-MT (augmented with 5 templates).
Protocols: \textbf{Standard} (training template),
\textbf{T-Stress} (sample one from $n\in\{1,\ldots,15\}$ semantics-preserving
templates), and \textbf{D-Test} (operator-based input mutations from
\citet{yang2025dynamicmmeval}; \suppsec{D}).
$\Delta\mathrm{T}$ and $\Delta\mathrm{D}$ denote the
Standard$\!\to\!$T-Stress and Standard$\!\to\!$D-Test gaps
(negative = degradation). All reported accuracies are means over 3 seeds.

\subsection{Template Stress Test}
\label{sec:exp_tstress}

Table~\ref{tab:main_results_comprehensive} reports Standard and T-Stress
accuracy, and Figure~\ref{fig:mean_tstress_gap} summarizes the
mean gap per method. On \textbf{general Exam VQA}, standard GRPO gains
in-template accuracy but loses sharply under template shift: on Qwen2.5-VL-7B
MathVista drops by $4.4$ and Olym-Phys by $1.6$, whereas PIRL is flat-to-positive
($\Delta\mathrm{T} = +0.3$ on MathVista, $+1.1$ on Olym-Phys) and matches or beats
the five-template GRPO-MT baseline (MathVista $+0.3$ vs.\
GRPO-MT $-0.7$). Absent an invariance constraint of the kind in (C2),
\suppprop{B.5} gives a mechanism for gaps that scale with the
unseen-template mass, which is the pattern GRPO and GRPO-MT display on Exam VQA.

\paragraph{High-stakes VQA.}
The Medical and Legal blocks of Table~\ref{tab:main_results_comprehensive}
target deployment-oriented domains, where an answer that flips under a reworded
instruction is of little use when the downstream decision carries operational
cost. On \emph{medical VQA}, GRPO degrades under template stress (Qwen2.5-VL-7B:
VQA-RAD $-2.3$, PathVQA $-6.7$) while PIRL stays near flat (VQA-RAD $+0.8$,
PathVQA $-0.8$); the exception is GMAI, where PIRL's drop ($-2.7$) exceeds
GRPO-MT ($-1.3$) without dominating. On \emph{legal VQA} (DocVQA), PIRL retains
accuracy on both base models ($\Delta\mathrm{T} = +0.2$ on Qwen2.5-VL-7B, $+0.5$
on Qwen3-VL-8B) whereas GRPO degrades ($-2.4$ and $-1.3$) and GRPO-MT recovers
only partly.

\subsection{Effect of Template Count}
\label{sec:exp_templatecount}

We sweep $N$, the number of different templates used in evaluation, on Exam VQA
(Fig.~\ref{fig:templates}). GRPO ID accuracy drops from $49.5$ at $N{=}1$ to
$47.7$ at $N{\ge}5$, qualitatively consistent with the mass-dependent
construction in \suppprop{B.5}; the proposition is not a fitted prediction of
this curve. GRPO-MT flattens it (ID $49.5$, OOD
$32.3$). PIRL keeps ID accuracy at $51.1$ and OOD accuracy increases up to
$N{=}10$ ($40.5\to 41.2$).

\subsection{Dynamic Evaluation}
\label{sec:exp_dtest_split}

Under operator-based input mutations~\citep{zhou2026unified}
(Table~\ref{tab:main_results_comprehensive_dynamic};
\suppsec{D}) on \textbf{general (Exam) VQA} PIRL
narrows the open-ended gap (MMK12 $\Delta\mathrm{D} = -1.1$ vs.\ GRPO $-3.6$ on
Qwen2.5-VL-7B), whereas on MathVista every
method collapses ($\sim$$-47$ to $-49$) because multi-choice questions are reconstructed into free-form QA rather than simple paraphrase.

\paragraph{High-stakes VQA.}
On \emph{medical VQA}, PIRL is consistently the least-degraded trained method
(Qwen2.5-VL-7B: VQA-RAD $\Delta\mathrm{D} = -7.3$ vs.\ GRPO $-12.1$; PathVQA
$-3.1$ vs.\ $-5.8$; GMAI $-7.2$ vs.\ $-7.6$). On \emph{legal VQA} (DocVQA),
document-layout mutations degrade every method by $\sim$$20$ points---a
task-altering input mutation rather than a rewording---with PIRL degrading least
($\Delta\mathrm{D} = -20.5$ and $-20.3$ vs.\ GRPO $-22.8$ and $-22.0$). The
document-VQA prompt-robustness claim therefore rests on template stress
(Section~\ref{sec:exp_tstress}), where the failure that matters is an answer that
flips under a reworded instruction.

\subsection{Ablation}
\label{sec:exp_ablation}

Table~\ref{tab:ablation_final_filled} ablates multi-template (MT),
adversarial prompts (ADV; embedding-space), and DTR on Exam VQA Avg.\ and
Olym-Phys. Full PIRL has the strongest T-Stress result ($\Delta\mathrm{T} =
+0.6$ vs.\ GRPO $-2.3$). On the aggregate D-Test the four configurations
cluster within noise ($\Delta\mathrm{D}$ from $-18.4$ to $-17.5$): the
average is dominated by the multiple-choice math benchmark (MathVista)
on which every method collapses
(Table~\ref{tab:main_results_comprehensive_dynamic}), so it does not separate
components. The component signal is therefore on T-Stress, where full PIRL
($+0.6$) and MT+DTR ($0.0$) sit within noise of each other while both clearly
beat GRPO ($-2.3$); this ablation does not cleanly isolate DTR's or the
adversary's individual contributions beyond MT.
We thus present DTR and the adversary as principled additions with no detectable cost
here, not as individually necessary for the leading-order T-Stress gain. On
Olym-Phys, PIRL improves $\Delta\mathrm{T}$ ($+1.1$) and D-Test accuracy over
GRPO, although its $\Delta\mathrm{D}$ is worse ($-4.1$ vs.\ $-3.1$) because
the higher Standard accuracy enlarges the absolute drop. The adversary edits the instruction side only; image-side
robustness is reported in \suppsec{C.1}
(Figure~\ref{fig:perturb_ablation}).

\subsection{Discussion}
\label{sec:exp_discussion}

Against the $\sim$$1.5$ pp inconclusive threshold, the PathVQA/MathVista wins
($-0.8, +0.3$ vs.\ GRPO $-6.7, -4.4$) clear it, while the GMAI gap ($-2.7$ vs.\
GRPO-MT $-1.3$) does not. On Pass@8
(Table~\ref{tab:pass8_results}) PIRL is flat vs.\ base and $+0.6$ to $+2.5$
vs.\ GRPO; the ceiling is set by the base, and PIRL's robustness objective is
consistent with evidence that RLVR often reweights existing probability mass
rather than enlarging the reasoning frontier~\citep{yue2025does,liu2025understanding}.

\section{Related Work}

\textbf{RL for reasoning.}
RLVR is widely used to improve reasoning in LLMs and
MLLMs~\citep{guo2025deepseek,meng2025mm,zhang2025survey,zhou2025reinforced},
but outcome-based optimization also exposes Goodhart-style failures where
policies exploit reward or evaluator
artifacts~\citep{teney2020value,yue2025does,zhou2025opening}, with reported biases such as
length and reward
hacking~\citep{hu2025reinforcep,wang2025causally,lu2025r,liu2025understanding}. While prior reasoning frameworks rely heavily on static environments, we study how a multiplicative verifier can reduce correctness-informative reward variation when format compliance is low.

\noindent\textbf{Robust RL and reward generalization.}
RL-tuned MLLMs overfit to response templates under imperfect
rewards~\citep{zhangimproving,hu2025reinforcep}. Robustness work uses RL to
search adversarial prompts or red-team~\citep{chen2024llm,lin2025r1,pan2025beyond}
or studies perception-side perturbations~\citep{li2025vision}; we instead
train with an in-loop adversary that perturbs the prompt embedding.
Reward-side methods (reward-prompt evolution, preference augmentation,
process rewards~\citep{kim2025toward,zhangimproving,yereward,wang2025adversarial,wang2024rlhfpoison,mao2025information})
are complementary for our framework.

\section{Conclusion}


We formalize the prompt-robustness gap through reward-signal entanglement
and population-level objective non-identifiability, and identify
Reward Decomposition (C1) and Invariance Control (C2) as the corresponding
design conditions. PIRL instantiates these conditions through a Dynamic
Trinary Reward, an embedding-space prompt adversary, and consistency
regularization. Experiments across multimodal reasoning benchmarks show that
PIRL reduces template-induced degradation while largely preserving accuracy.
Extending robust agentic RL to agentic model-routing systems~\cite{zhou2026agent},
where prompt variation can affect both route selection and downstream
generation, is our next step.

\section*{Limitations}

\textbf{Scope of the invariance objective.}
PIRL promotes invariance within a continuous instruction-embedding
neighborhood, a tractable relaxation of the semantic equivalence class.
Although individual perturbations need not decode to literal prompts, this
design enables efficient exploration beyond a finite template pool, while the
token-level KL provides a practical surrogate for answer-level invariance.

\noindent\textbf{Theoretical scope and scaling considerations.}
Our non-identifiability and C1/C2 results are population-level under coverage
assumptions that isolate the underlying mechanism. Finite-sample learnability
guarantees and scaling beyond the 7--8B models evaluated here are promising
directions that the framework is well positioned to support.

{\small
\bibliographystyle{bibstyle}
\bibliography{references}

\begin{thebibliography}{39}
\providecommand{\natexlab}[1]{#1}
\providecommand{\url}[1]{\texttt{#1}}
\expandafter\ifx\csname urlstyle\endcsname\relax
  \providecommand{\doi}[1]{doi: #1}\else
  \providecommand{\doi}{doi: \begingroup \urlstyle{rm}\Url}\fi

\bibitem[Bai et~al.(2025{\natexlab{a}})Bai, Cai, Chen, Chen, Chen, Cheng, Deng,
  Ding, Gao, Ge, Ge, Guo, Huang, Huang, Huang, Hui, Jiang, Li, Li, Li, Li, Lin,
  Lin, Liu, Liu, Liu, Liu, Liu, Liu, Lu, Luo, Lv, Men, Meng, Ren, Ren, Song,
  Sun, Tang, Tu, Wan, Wang, Wang, Wang, Wang, Xie, Xu, Xu, Xu, Yang, Yang,
  Yang, Yang, Yu, Zhang, Zhang, Zhang, Zheng, Zhong, Zhou, Zhou, Zhou, Zhu, and
  Zhu]{bai2025qwen3vl}
Shuai Bai, Yuxuan Cai, Ruizhe Chen, Keqin Chen, Xionghui Chen, Zesen Cheng,
  Lianghao Deng, Wei Ding, Chang Gao, Chunjiang Ge, Wenbin Ge, Zhifang Guo,
  Qidong Huang, Jie Huang, Fei Huang, Binyuan Hui, Shutong Jiang, Zhaohai Li,
  Mingsheng Li, Mei Li, Kaixin Li, Zicheng Lin, Junyang Lin, Xuejing Liu,
  Jiawei Liu, Chenglong Liu, Yang Liu, Dayiheng Liu, Shixuan Liu, Dunjie Lu,
  Ruilin Luo, Chenxu Lv, Rui Men, Lingchen Meng, Xuancheng Ren, Xingzhang Ren,
  Sibo Song, Yuchong Sun, Jun Tang, Jianhong Tu, Jianqiang Wan, Peng Wang,
  Pengfei Wang, Qiuyue Wang, Yuxuan Wang, Tianbao Xie, Yiheng Xu, Haiyang Xu,
  Jin Xu, Zhibo Yang, Mingkun Yang, Jianxin Yang, An Yang, Bowen Yu, Fei Zhang,
  Hang Zhang, Xi Zhang, Bo Zheng, Humen Zhong, Jingren Zhou, Fan Zhou, Jing
  Zhou, Yuanzhi Zhu, and Ke Zhu.
\newblock Qwen3-vl technical report.
\newblock \emph{arXiv preprint arXiv:2511.21631}, 2025{\natexlab{a}}.

\bibitem[Bai et~al.(2025{\natexlab{b}})Bai, Chen, Liu, Wang, Ge, Song, Dang,
  Wang, Wang, Tang, Zhong, Zhu, Yang, Li, Wan, Wang, Ding, Fu, Xu, Ye, Zhang,
  Xie, Cheng, Zhang, Yang, Xu, and Lin]{bai2025qwen25vl}
Shuai Bai, Keqin Chen, Xuejing Liu, Jialin Wang, Wenbin Ge, Sibo Song, Kai
  Dang, Peng Wang, Shijie Wang, Jun Tang, Humen Zhong, Yuanzhi Zhu, Mingkun
  Yang, Zhaohai Li, Jianqiang Wan, Pengfei Wang, Wei Ding, Zheren Fu, Yiheng
  Xu, Jiabo Ye, Xi Zhang, Tianbao Xie, Zesen Cheng, Hang Zhang, Zhibo Yang,
  Haiyang Xu, and Junyang Lin.
\newblock Qwen2.5-vl technical report.
\newblock \emph{arXiv preprint arXiv:2502.13923}, 2025{\natexlab{b}}.

\bibitem[Camburu et~al.(2018)Camburu, Rockt{\"a}schel, Lukasiewicz, and
  Blunsom]{camburu2018snli}
Oana-Maria Camburu, Tim Rockt{\"a}schel, Thomas Lukasiewicz, and Phil Blunsom.
\newblock e-{SNLI}: Natural language inference with natural language
  explanations.
\newblock \emph{Advances in Neural Information Processing Systems}, 31, 2018.

\bibitem[Chen et~al.(2024{\natexlab{a}})Chen, Ye, Wang, Li, Deng, Li, Li, Duan,
  Huang, Su, Wang, Zhang, Fu, Cai, Zhuang, Seibel, Qiao, and
  He]{chen2024gmaimmbench}
Pengcheng Chen, Jin Ye, Guoan Wang, Yanjun Li, Zhongying Deng, Wei Li, Tianbin
  Li, Haodong Duan, Ziyan Huang, Yanzhou Su, Benyou Wang, Shaoting Zhang, Bin
  Fu, Jianfei Cai, Bohan Zhuang, Eric~J. Seibel, Yu Qiao, and Junjun He.
\newblock {GMAI-MMBench}: A comprehensive multimodal evaluation benchmark
  towards general medical ai.
\newblock \emph{Advances in Neural Information Processing Systems},
  37:\penalty0 94327--94427, 2024{\natexlab{a}}.

\bibitem[Chen et~al.(2024{\natexlab{b}})Chen, Nie, Guo, and Zhang]{chen2024llm}
Xuan Chen, Yuzhou Nie, Wenbo Guo, and Xiangyu Zhang.
\newblock When llm meets drl: Advancing jailbreaking efficiency via drl-guided
  search.
\newblock \emph{Advances in Neural Information Processing Systems},
  37:\penalty0 26814--26845, 2024{\natexlab{b}}.

\bibitem[{DeepSeek-AI}(2025)]{guo2025deepseek}
{DeepSeek-AI}.
\newblock {DeepSeek-R1}: Incentivizing reasoning capability in llms via
  reinforcement learning.
\newblock \emph{arXiv preprint arXiv:2501.12948}, 2025.

\bibitem[He et~al.(2024)He, Luo, Bai, Hu, Thai, Shen, Hu, Han, Huang, Zhang,
  Liu, Qi, Liu, and Sun]{he2024olympiadbench}
Chaoqun He, Renjie Luo, Yuzhuo Bai, Shengding Hu, Zhen~Leng Thai, Junhao Shen,
  Jinyi Hu, Xu Han, Yujie Huang, Yuxiang Zhang, Jie Liu, Lei Qi, Zhiyuan Liu,
  and Maosong Sun.
\newblock Olympiadbench: A challenging benchmark for promoting agi with
  olympiad-level bilingual multimodal scientific problems.
\newblock In \emph{Proceedings of the 62nd Annual Meeting of the Association
  for Computational Linguistics (Volume 1: Long Papers)}, pages 3828--3850,
  2024.

\bibitem[He et~al.(2020)He, Zhang, Mou, Xing, and Xie]{he2020pathvqa}
Xuehai He, Yichen Zhang, Luntian Mou, Eric Xing, and Pengtao Xie.
\newblock Pathvqa: 30000+ questions for medical visual question answering.
\newblock \emph{arXiv preprint arXiv:2003.10286}, 2020.

\bibitem[Hu et~al.(2025)Hu, Liu, Xu, and Shen]{hu2025reinforcep}
Jian Hu, Jason~Klein Liu, Haotian Xu, and Wei Shen.
\newblock Reinforce++: Stabilizing critic-free policy optimization with global
  advantage normalization.
\newblock \emph{arXiv preprint arXiv:2501.03262}, 2025.

\bibitem[Kim et~al.(2025)Kim, Park, Raheja, Kim, and Kang]{kim2025toward}
Zae~Myung Kim, Chanwoo Park, Vipul Raheja, Suin Kim, and Dongyeop Kang.
\newblock Toward evaluative thinking: Meta policy optimization with evolving
  reward models.
\newblock \emph{arXiv preprint arXiv:2504.20157}, 2025.

\bibitem[{LASA Team} et~al.(2025){LASA Team}, Xu, Chan, Li, Aljunied, Yuan,
  Wang, Xiao, Chen, Liu, Li, Sun, Shen, Wang, Tan, Zhao, Xu, Zhang, and
  Rong]{xu2025lingshu}
{LASA Team}, Weiwen Xu, Hou~Pong Chan, Long Li, Mahani Aljunied, Ruifeng Yuan,
  Jianyu Wang, Chenghao Xiao, Guizhen Chen, Chaoqun Liu, Zhaodonghui Li, Yu
  Sun, Junao Shen, Chaojun Wang, Jie Tan, Deli Zhao, Tingyang Xu, Hao Zhang,
  and Yu Rong.
\newblock Lingshu: A generalist foundation model for unified multimodal medical
  understanding and reasoning.
\newblock \emph{arXiv preprint arXiv:2506.07044}, 2025.

\bibitem[Lau et~al.(2018)Lau, Gayen, Ben~Abacha, and
  Demner-Fushman]{lau2018dataset}
Jason~J. Lau, Soumya Gayen, Asma Ben~Abacha, and Dina Demner-Fushman.
\newblock A dataset of clinically generated visual questions and answers about
  radiology images.
\newblock \emph{Scientific Data}, 5\penalty0 (1):\penalty0 180251, 2018.

\bibitem[Li et~al.(2025)Li, Wei, Zheng, Huang, Kong, Sun, and
  Huang]{li2025vision}
Yuting Li, Lai Wei, Kaipeng Zheng, Jingyuan Huang, Linghe Kong, Lichao Sun, and
  Weiran Huang.
\newblock Vision matters: Simple visual perturbations can boost multimodal math
  reasoning.
\newblock \emph{arXiv preprint arXiv:2506.09736}, 2025.

\bibitem[Lin et~al.(2025)Lin, Su, Li, and Qian]{lin2025r1}
Jiayi Lin, Liangcai Su, Junzhe Li, and Chenxiong Qian.
\newblock R1-fuzz: Specializing language models for textual fuzzing via
  reinforcement learning.
\newblock \emph{arXiv preprint arXiv:2509.20384}, 2025.

\bibitem[Liu et~al.(2025)Liu, Chen, Li, Qi, Pang, Du, Lee, and
  Lin]{liu2025understanding}
Zichen Liu, Changyu Chen, Wenjun Li, Penghui Qi, Tianyu Pang, Chao Du, Wee~Sun
  Lee, and Min Lin.
\newblock Understanding r1-zero-like training: A critical perspective.
\newblock \emph{arXiv preprint arXiv:2503.20783}, 2025.

\bibitem[Lu et~al.(2024)Lu, Bansal, Xia, Liu, Li, Hajishirzi, Cheng, Chang,
  Galley, and Gao]{lu2024mathvista}
Pan Lu, Hritik Bansal, Tony Xia, Jiacheng Liu, Chunyuan Li, Hannaneh
  Hajishirzi, Hao Cheng, Kai-Wei Chang, Michel Galley, and Jianfeng Gao.
\newblock Mathvista: Evaluating mathematical reasoning of foundation models in
  visual contexts.
\newblock In \emph{The Twelfth International Conference on Learning
  Representations}, 2024.

\bibitem[Lu et~al.(2025)Lu, Wang, Guo, He, Tang, Gui, Huang, Cao, Wang, and
  Cai]{lu2025r}
Yi Lu, Jianing Wang, Linsen Guo, Wei He, Hongyin Tang, Tao Gui, Xuanjing Huang,
  Xuezhi Cao, Wei Wang, and Xunliang Cai.
\newblock R-horizon: How far can your large reasoning model really go in
  breadth and depth?
\newblock \emph{arXiv preprint arXiv:2510.08189}, 2025.

\bibitem[Mao et~al.(2025)Mao, Xu, Zhang, Zhang, and Bai]{mao2025information}
Liyuan Mao, Haoran Xu, Amy Zhang, Weinan Zhang, and Chenjia Bai.
\newblock Information-theoretic reward decomposition for generalizable rlhf.
\newblock \emph{arXiv preprint arXiv:2504.06020}, 2025.

\bibitem[Mathew et~al.(2021)Mathew, Karatzas, and Jawahar]{mathew2021docvqa}
Minesh Mathew, Dimosthenis Karatzas, and C.~V. Jawahar.
\newblock Docvqa: A dataset for vqa on document images.
\newblock In \emph{Proceedings of the IEEE/CVF Winter Conference on
  Applications of Computer Vision}, pages 2199--2208, 2021.

\bibitem[Meng et~al.(2025)Meng, Du, Liu, Zhou, Lu, Fu, Han, Shi, Wang, He,
  Zhang, Luo, Qiao, Zhang, and Shao]{meng2025mm}
Fanqing Meng, Lingxiao Du, Zongkai Liu, Zhixiang Zhou, Quanfeng Lu, Daocheng
  Fu, Tiancheng Han, Botian Shi, Wenhai Wang, Junjun He, Kaipeng Zhang, Ping
  Luo, Yu Qiao, Qiaosheng Zhang, and Wenqi Shao.
\newblock {MM-Eureka}: Exploring the frontiers of multimodal reasoning with
  rule-based reinforcement learning.
\newblock \emph{arXiv preprint arXiv:2503.07365}, 2025.

\bibitem[Pan et~al.(2025)Pan, Jian, Hager, Zhang, Liu, Jungmann, Li, You, Wu,
  Zhu, Liu, Liu, Bubeck, Wachinger, Chen, Gong, Ouyang, Kaissis, Wiestler, and
  Rueckert]{pan2025beyond}
Jiazhen Pan, Bailiang Jian, Paul Hager, Yundi Zhang, Che Liu, Friedrike
  Jungmann, Hongwei~Bran Li, Chenyu You, Junde Wu, Jiayuan Zhu, Fenglin Liu,
  Yuyuan Liu, Niklas Bubeck, Christian Wachinger, Chen Chen, Zhenyu Gong, Cheng
  Ouyang, Georgios Kaissis, Benedikt Wiestler, and Daniel Rueckert.
\newblock Beyond benchmarks: Dynamic, automatic and systematic red-teaming
  agents for trustworthy medical language models.
\newblock \emph{arXiv preprint arXiv:2508.00923}, 2025.

\bibitem[Sagawa et~al.(2020)Sagawa, Koh, Hashimoto, and
  Liang]{sagawa2020distributionally}
Shiori Sagawa, Pang~Wei Koh, Tatsunori~B. Hashimoto, and Percy Liang.
\newblock Distributionally robust neural networks for group shifts: On the
  importance of regularization for worst-case generalization.
\newblock In \emph{International Conference on Learning Representations}, 2020.

\bibitem[Schulman et~al.(2017)Schulman, Wolski, Dhariwal, Radford, and
  Klimov]{schulman2017proximal}
John Schulman, Filip Wolski, Prafulla Dhariwal, Alec Radford, and Oleg Klimov.
\newblock Proximal policy optimization algorithms.
\newblock \emph{arXiv preprint arXiv:1707.06347}, 2017.

\bibitem[Shao et~al.(2024)Shao, Wang, Zhu, Xu, Song, Bi, Zhang, Zhang, Li, Wu,
  and Guo]{shao2024deepseekmath}
Zhihong Shao, Peiyi Wang, Qihao Zhu, Runxin Xu, Junxiao Song, Xiao Bi, Haowei
  Zhang, Mingchuan Zhang, Y.~K. Li, Y. Wu, and Daya Guo.
\newblock {DeepSeekMath}: Pushing the limits of mathematical reasoning in open
  language models.
\newblock \emph{arXiv preprint arXiv:2402.03300}, 2024.

\bibitem[Teney et~al.(2020)Teney, Kafle, Shrestha, Abbasnejad, Kanan, and
  van~den Hengel]{teney2020value}
Damien Teney, Kushal Kafle, Robik Shrestha, Ehsan Abbasnejad, Christopher
  Kanan, and Anton van~den Hengel.
\newblock On the value of out-of-distribution testing: An example of
  {Goodhart}'s law.
\newblock In \emph{Advances in Neural Information Processing Systems}, pages
  407--417, 2020.

\bibitem[Van~Landeghem et~al.(2023)Van~Landeghem, Tito, Borchmann, Pietruszka,
  J{\'o}ziak, Powalski, Jurkiewicz, Coustaty, Ackaert, Valveny, Blaschko,
  Moens, and Stanis{\l}awek]{vanlandeghem2023dude}
Jordy Van~Landeghem, Rub{\`e}n Tito, {\L}ukasz Borchmann, Micha{\l} Pietruszka,
  Pawe{\l} J{\'o}ziak, Rafa{\l} Powalski, Dawid Jurkiewicz, Micka{\"e}l
  Coustaty, Bertrand Ackaert, Ernest Valveny, Matthew~B. Blaschko, Sien Moens,
  and Tomasz Stanis{\l}awek.
\newblock Document understanding dataset and evaluation (dude).
\newblock In \emph{Proceedings of the IEEE/CVF International Conference on
  Computer Vision}, pages 19528--19540, 2023.

\bibitem[Wang et~al.(2024)Wang, Wu, Chen, Vorobeychik, and
  Xiao]{wang2024rlhfpoison}
Jiongxiao Wang, Junlin Wu, Muhao Chen, Yevgeniy Vorobeychik, and Chaowei Xiao.
\newblock Rlhfpoison: Reward poisoning attack for reinforcement learning with
  human feedback in large language models.
\newblock In \emph{Proceedings of the 62nd Annual Meeting of the Association
  for Computational Linguistics (Volume 1: Long Papers)}, pages 2551--2570,
  2024.

\bibitem[Wang et~al.(2025{\natexlab{a}})Wang, Huang, Zhou, Luo, Guo, and
  Zhang]{wang2025causally}
Xiangqi Wang, Yue Huang, Yujun Zhou, Xiaonan Luo, Kehan Guo, and Xiangliang
  Zhang.
\newblock Causally-enhanced reinforcement policy optimization.
\newblock \emph{arXiv preprint arXiv:2509.23095}, 2025{\natexlab{a}}.

\bibitem[Wang et~al.(2025{\natexlab{b}})Wang, Wang, Xi, Tang, Zhu, Wei, Chen,
  Yang, Zhang, Lu, Niu, Mao, Li, Xiong, Hu, and Yang]{wang2025adversarial}
Yuanfu Wang, Pengyu Wang, Chenyang Xi, Bo Tang, Junyi Zhu, Wenqiang Wei, Chen
  Chen, Chao Yang, Jingfeng Zhang, Chaochao Lu, Yijun Niu, Keming Mao, Zhiyu
  Li, Feiyu Xiong, Jie Hu, and Mingchuan Yang.
\newblock Adversarial preference learning for robust llm alignment.
\newblock \emph{arXiv preprint arXiv:2505.24369}, 2025{\natexlab{b}}.

\bibitem[Yang et~al.(2025)Yang, Zhang, Shao, Zhang, Bin, Wang, and
  Luo]{yang2025dynamicmmeval}
Yue Yang, Shuibai Zhang, Wenqi Shao, Kaipeng Zhang, Yi Bin, Yu Wang, and Ping
  Luo.
\newblock Dynamic multimodal evaluation with flexible complexity by
  vision-language bootstrapping.
\newblock In \emph{International Conference on Learning Representations}, 2025.

\bibitem[Ye et~al.(2025)Ye, Agarwal, Liu, Joshi, Velury, Le, Tan, and
  Liu]{yereward}
Ziyu Ye, Rishabh Agarwal, Tianqi Liu, Rishabh Joshi, Sarmishta Velury, Quoc~V
  Le, Qijun Tan, and Yuan Liu.
\newblock Reward-guided prompt evolving in reinforcement learning for llms.
\newblock In \emph{Forty-second International Conference on Machine Learning},
  2025.

\bibitem[Yue et~al.(2025)Yue, Chen, Lu, Zhao, Wang, Yue, Song, and
  Huang]{yue2025does}
Yang Yue, Zhiqi Chen, Rui Lu, Andrew Zhao, Zhaokai Wang, Yang Yue, Shiji Song,
  and Gao Huang.
\newblock Does reinforcement learning really incentivize reasoning capacity in
  llms beyond the base model?
\newblock \emph{Advances in Neural Information Processing Systems},
  38:\penalty0 57654--57689, 2025.

\bibitem[Zhang et~al.(2025{\natexlab{a}})Zhang, Zuo, He, Sun, Liu, Jiang, Fan,
  Tian, Jia, Li, Fu, Lv, Zhang, Zeng, Qu, Li, Wang, Wang, Long, Liu, Xu, Ma,
  Zhu, Hua, Liu, Li, Chen, Qu, Li, Chen, Yuan, Gao, Li, Ma, Cui, Liu, Qi, Ding,
  and Zhou]{zhang2025survey}
Kaiyan Zhang, Yuxin Zuo, Bingxiang He, Youbang Sun, Runze Liu, Che Jiang,
  Yuchen Fan, Kai Tian, Guoli Jia, Pengfei Li, Yu Fu, Xingtai Lv, Yuchen Zhang,
  Sihang Zeng, Shang Qu, Haozhan Li, Shijie Wang, Yuru Wang, Xinwei Long,
  Fangfu Liu, Xiang Xu, Jiaze Ma, Xuekai Zhu, Ermo Hua, Yihao Liu, Zonglin Li,
  Huayu Chen, Xiaoye Qu, Yafu Li, Weize Chen, Zhenzhao Yuan, Junqi Gao, Dong
  Li, Zhiyuan Ma, Ganqu Cui, Zhiyuan Liu, Biqing Qi, Ning Ding, and Bowen Zhou.
\newblock A survey of reinforcement learning for large reasoning models.
\newblock \emph{arXiv preprint arXiv:2509.08827}, 2025{\natexlab{a}}.

\bibitem[Zhang et~al.(2025{\natexlab{b}})Zhang, Xu, Du, Cao, Sun, and
  Yu]{zhangimproving}
Zhilong Zhang, Tian Xu, Xinghao Du, Xingchen Cao, Yihao Sun, and Yang Yu.
\newblock Improving reward model generalization from adversarial process
  enhanced preferences.
\newblock In \emph{Forty-second International Conference on Machine Learning},
  2025{\natexlab{b}}.

\bibitem[Zhou et~al.(2025{\natexlab{a}})Zhou, Qiu, Chen, Wang, Yang, Xu, and
  Qiu]{zhou2025reinforced}
Guanghao Zhou, Panjia Qiu, Cen Chen, Jie Wang, Zheming Yang, Jian Xu, and
  Minghui Qiu.
\newblock Reinforced mllm: A survey on rl-based reasoning in multimodal large
  language models.
\newblock \emph{arXiv preprint arXiv:2504.21277}, 2025{\natexlab{a}}.

\bibitem[Zhou et~al.(2025{\natexlab{b}})Zhou, Peng, Song, Li, Xu, Yang, Guo,
  Zhang, Lin, He, et~al.]{zhou2025opening}
Pengfei Zhou, Xiaopeng Peng, Jiajun Song, Chuanhao Li, Zhaopan Xu, Yue Yang,
  Ziyao Guo, Hao Zhang, Yuqi Lin, Yefei He, et~al.
\newblock {Opening}: A comprehensive benchmark for judging open-ended
  interleaved image-text generation.
\newblock In \emph{2025 IEEE/CVF Conference on Computer Vision and Pattern
  Recognition (CVPR)}, pages 56--66. IEEE, 2025{\natexlab{b}}.

\bibitem[Zhou et~al.(2026{\natexlab{a}})Zhou, Peng, Zhang, Xu, Ai, Qiu, Zhao,
  Song, Li, Tang, et~al.]{zhou2026mdk12}
Pengfei Zhou, Xiaopeng Peng, Fanrui Zhang, Zhaopan Xu, Jiaxin Ai, Yansheng Qiu,
  Wangbo Zhao, Jiajun Song, Chuanhao Li, Weidong Tang, et~al.
\newblock {MDK12-Bench}: A multi-discipline benchmark for evaluating reasoning
  in multimodal large language models.
\newblock In \emph{Proceedings of the AAAI Conference on Artificial
  Intelligence}, pages 28982--28990, 2026{\natexlab{a}}.

\bibitem[Zhou et~al.(2026{\natexlab{b}})Zhou, Song, Tang, Ma, Peng, Si, Xu,
  Song, Miao, Qian, et~al.]{zhou2026unified}
Pengfei Zhou, Jiajun Song, Zhiwei Tang, Yixing Ma, Xiaopeng Peng, Donghui Si,
  Yuhang Xu, Huiqi Song, Yiyuan Miao, Yichen Qian, et~al.
\newblock {Unified Hallucination Fuzzing for Multimodal Large Language Models}.
\newblock \emph{arXiv preprint arXiv:2608.07525}, 2026{\natexlab{b}}.

\bibitem[Zhou et~al.(2026{\natexlab{c}})Zhou, Tang, Ma, Tang, Han, Wan, Meng,
  Wang, Zhuang, Zhao, et~al.]{zhou2026agent}
Pengfei Zhou, Zhiwei Tang, Yixing Ma, Jiasheng Tang, Yizeng Han, Zhenglin Wan,
  Fanqing Meng, Wei Wang, Bohan Zhuang, Wangbo Zhao, et~al.
\newblock {Agent-as-a-Router}: Agentic model routing for coding tasks.
\newblock \emph{arXiv preprint arXiv:2606.22902}, 2026{\natexlab{c}}.

\end{thebibliography}
}



\appendix

\setcounter{table}{3}
\setcounter{figure}{3}

\section{Formal Definitions of (C1) and (C2)}
\label{sec:appendix_defs}

We give the formal versions of the two design conditions stated informally
in \mainsec{2.4}.
Throughout, total variation uses the standard convention
$\mathrm{TV}(P,Q)=\sup_A |P(A)-Q(A)|$, equivalently
$\frac12\|P-Q\|_1$ on countable spaces.
All transformation distributions in the population statements are probability
measures on $(\mathcal{T},\Sigma_{\mathcal{T}})$; finite-pool distributions are
viewed as zero-extended measures on this same space.

\begin{definition}[(C1) Reward Decomposition]
  \label{def:rd}
  The reward $R$ admits an additive decomposition
  $R = R_t + \alpha_{\mathrm{rew}} R_s$ with
  fixed $\alpha_{\mathrm{rew}}>0$, $R_t \in \{-1,0,+1\}$ format-aware, and
  $R_s \in [0,1]$ a bounded measurable semantic score. Fix the stochastic rollout-law
  class on which the condition is asserted, and let $\widetilde V_c$ denote the
  reward-side correctness indicator used for the bucket decomposition: for the
  binary verifier $\widetilde V_c=V_c$, while for extraction-based DTR
  $\widetilde V_c=\mathbf{1}\{\mathcal{A}_{\tau}(\bm{y})=a^*\}$.
  There exist constants
  $m,c>0$, independent of the $p_f$ values under consideration, such that for
  every rollout distribution in that class, every joint bucket
  $(\widetilde V_c,V_f)=(v,v')$ with positive probability satisfies
  $\mathrm{Var}(R_s\mid \widetilde V_c{=}v,V_f{=}v') \ge m$, and every format bucket with
  positive probability satisfies $\mathrm{Var}(R \mid V_f = v') \ge c$.
\end{definition}

\begin{definition}[(C2) Invariance Control (answer-level)]
  \label{def:ie}
  Let $\mathcal{C}(\bm{x})\subseteq\mathcal{T}$ be a finite
  semantics-preserving transformation pool for the fixed task instance
  $\bm{x}$, with every $\tau\in\mathcal{C}(\bm{x})$ satisfying
  $\tau(\bm{x})\in\llbracket\bm{x}\rrbracket$. Let
  $\mathcal{Y}_{\text{ans}}$ denote the answer space, and equip
  $\mathcal{Y}_{\text{ans}}\cup\{\emptyset\}$ with the discrete
  sigma-algebra, with $\emptyset\notin\mathcal{Y}_{\text{ans}}$.
  For each transformation $\tau$, let
  $\mathcal{A}_{\tau}(\bm{y}) := \mathcal{A}(\bm{y};d_\tau)
  \in\mathcal{Y}_{\text{ans}}\cup\{\emptyset\}$ be the measurable deterministic
  answer-extraction map using the delimiter requested by $\tau$ (the extractor
  used in \maineq{5}). Write
  \begin{equation}
    \pi^{\mathcal{A}_{\tau}}_\theta(\cdot\mid\tau(\bm{x}))
    :=
    (\mathcal{A}_{\tau})_{\#}\pi_\theta(\cdot\mid\tau(\bm{x}))
  \end{equation}
  for the induced prompt-conditioned answer marginal, where the countable-token
  special case gives
  $\pi^{\mathcal{A}_{\tau}}_\theta(a\mid\tau(\bm{x}))
  =\sum_{\bm{y}:\mathcal{A}_{\tau}(\bm{y})=a}
    \pi_\theta(\bm{y}\mid\tau(\bm{x}))$,
  $P_\tau^\theta:=\pi^{\mathcal{A}_{\tau}}_\theta(\cdot\mid\tau(\bm{x}))$, and
  $\mathcal{F}_\tau$ denotes the format requested by $\tau$. For an
  extractor-consistent format family, if a response in $\mathcal{F}_\tau$
  encodes answer $a$, then $\mathcal{A}_{\tau}$ returns $a$; if the response
  is in a format family disjoint from $\mathcal{F}_\tau$, then
  $\mathcal{A}_{\tau}$ returns $\emptyset$. The C1/C2 consequences proof uses
  this delimiter-separated extractor convention.
  For an $f$-divergence $D$ admitting a Pinsker-type bound
  $\mathrm{TV}(P, Q) \le \kappa_D \sqrt{D(P\|Q)}$
  (for the natural-log KL used here, $\kappa_D = 1/\sqrt 2$),
  the training procedure fixes a non-degenerate coverage
  distribution $\rho_{\bm{x}}\in\Delta(\mathcal{C}(\bm{x}))$ and imposes
  \begin{equation}
    \mathbb{E}_{\tau_1, \tau_2 \sim \rho_{\bm{x}}}
    \!\left[D\!\left(P_{\tau_1}^{\theta}\,\middle\|\,P_{\tau_2}^{\theta}\right)\right]
    \;\le\; \varepsilon.
  \end{equation}
  The non-degeneracy requirement means
  $|\{\tau\in\mathcal{C}(\bm{x}):\rho_{\bm{x}}(\{\tau\})>0\}|\ge2$,
  ruling out the vacuous point-mass choice $\rho_{\bm{x}}=\delta_{\tau}$.
  This condition only controls transformations receiving coverage mass.
\end{definition}

\begin{lemma}[Data-processing inequality for the answer marginal]
  \label{lem:dpi_surrogate}
  Let $D$ be any $f$-divergence and let $\mathcal{A}_{\#}P$ denote the
  push-forward of a sequence distribution $P$ under a fixed deterministic
  extraction map $\mathcal{A}$. For any two prompts
  $\tau_1(\bm{x}), \tau_2(\bm{x})$, write
  $P_i=\pi_\theta(\cdot\mid\tau_i(\bm{x}))$. Then
  \begin{equation}
    D\!\left(\mathcal{A}_{\#}P_1\,\middle\|\,\mathcal{A}_{\#}P_2\right)
    \;\le\; D\!\left(P_1\,\middle\|\,P_2\right),
  \end{equation}
  where the right-hand side is the divergence between full sequence
  distributions.
\end{lemma}

\begin{proof}
$\bm{y} \mapsto \mathcal{A}(\bm{y})$ is a deterministic data-processing
channel; the data processing inequality for $f$-divergences gives the bound.
\end{proof}

\paragraph{Relation to the per-token estimator used in
\maineq{8}.} The penalty $\mathcal{L}_{\mathrm{cons}}$ is computed
\emph{per token} and length-normalized. With
$p_t^{\tau}:=\pi_\theta(\cdot\mid\bm{y}_{<t},\tau_\phi(\bm{x}))$ and
$p_t^{\bm{x}}:=\pi_\theta(\cdot\mid\bm{y}_{<t},\bm{x})$,
\begin{equation}
\begin{aligned}
\mathcal{L}_{\mathrm{cons}}(\theta;\bm{x},\phi)
&=
\frac{1}{T}\sum_{t=1}^{T}
\mathbb{E}_{\bm{y}_{<t}\sim\nu_t}  D_{\mathrm{KL}}\!\left(
  p_t^{\tau}\,\middle\|\,\operatorname{sg}(p_t^{\bm{x}})
\right),
\end{aligned}
\end{equation}
not on the full sequence distribution. The stop-gradient does not change the
KL value, only the update direction. At a fixed $\theta$ and $\phi$, KL obeys a
chain rule: if $\nu_t$ is the prefix law induced by the perturbed-prompt policy
under a fixed truncation/stopping convention, the full-sequence
$D_{\mathrm{KL}}(\pi_\theta(\cdot\mid\tau_\phi(\bm{x}))\,\|\,\pi_\theta(\cdot\mid\bm{x}))$
equals the unnormalized per-token sum averaged over prefixes drawn from
$\pi_\theta(\cdot\mid\tau_\phi(\bm{x}))$. Thus, for that sampled-prefix choice,
$T\mathcal{L}_{\mathrm{cons}}$ estimates this truncated sequence KL;
$\mathcal{L}_{\mathrm{cons}}$ itself is a length-normalized surrogate.
Teacher-forced prefixes use a different $\nu_t$ and do not give this equality
exactly; they are an
implementation surrogate for the same local consistency pressure.
We therefore do not claim that $\mathcal{L}_{\mathrm{cons}}$ is a uniform
upper bound on the right-hand side of Lemma~\ref{lem:dpi_surrogate}. The
lemma applies only when the same deterministic extractor is used on both
sequence distributions; (C2) uses the prompt-conditioned family
$\mathcal{A}_{\tau}$. Thus the per-token penalty is a tractable,
differentiable surrogate rather than a formal certificate of (C2). The
empirical question, whether optimizing the per-token surrogate decreases the
answer-level divergence, remains open and is consistent with the observed
robustness in \mainsec{4}.

\section{Full Proofs}
\label{sec:appendix_proofs}

\subsection{Binary-verifier within-format reward-variance share (Theorem~\ref{thm:snr})}
\label{sec:appendix_snr}

\begin{theorem}[Binary-verifier within-format reward-variance share]
  \label{thm:snr}
  Under $V_f \perp V_c$ with $p_f, p_c \in (0,1)$, the variance of $V$ in
  \maineq{2} decomposes via the law of total variance as
  \begin{equation}
  \begin{aligned}
    \mathrm{Var}(V)
    &= \underbrace{p_f p_c (1{-}p_c)}_{\text{within-format}}
      + \underbrace{p_c^{2} p_f (1{-}p_f)}_{\text{between-format}} \\
    &= p_f p_c (1{-}p_f p_c).
    \label{eq:variance_decomp}
  \end{aligned}
  \end{equation}
  The \emph{within-format reward-variance share} is
  \begin{equation}
    \Phi_{\mathrm{wf}}(p_f, p_c)
    := \frac{\mathbb{E}[\mathrm{Var}(V\mid V_f)]}{\mathrm{Var}(V)}
    = \frac{1 - p_c}{1 - p_f p_c},
    \label{eq:phi_wf}
  \end{equation}
  strictly increasing in $p_f$ on $(0,1)$, with
  $\lim_{p_f\uparrow 1}\Phi_{\mathrm{wf}}=1$ and
  $\lim_{p_f\downarrow0}\Phi_{\mathrm{wf}}=1-p_c$.
\end{theorem}
For $G$ independent rollouts from the same prompt distribution, the expected
count of format-valid samples is $Gp_f$.

\noindent\textit{Proof.}
By the law of total variance applied to $V = V_f V_c$ with $V_f \perp V_c$,
we decompose the variability of $V$ into the conditional variance given $V_f$
and the variance induced by the randomness of $V_f$:
\begin{equation}
  \mathrm{Var}(V) \;=\; \mathbb{E}\bigl[\mathrm{Var}(V \mid V_f)\bigr] + \mathrm{Var}\bigl(\mathbb{E}[V \mid V_f]\bigr).
\end{equation}
The conditional moments are direct: since $V = V_c$ when $V_f = 1$ and $V = 0$
when $V_f = 0$,
\begin{equation}
\begin{aligned}
  \mathbb{E}[V \mid V_f = 1] &= \mathbb{E}[V_c] = p_c, \\
  \mathbb{E}[V \mid V_f = 0] &= 0, \\
  \mathrm{Var}(V \mid V_f = 1) &= \mathrm{Var}(V_c) = p_c(1 - p_c), \\
  \mathrm{Var}(V \mid V_f = 0) &= 0.
\end{aligned}
\end{equation}
Therefore
\begin{equation}
\begin{aligned}
  \mathbb{E}\bigl[\mathrm{Var}(V \mid V_f)\bigr]
    &= p_f \cdot p_c(1 - p_c), \\
  \mathrm{Var}\bigl(\mathbb{E}[V \mid V_f]\bigr)
    &= p_f \cdot p_c^{2} - (p_f \, p_c)^{2} \\
    &= p_c^{2} \, p_f(1 - p_f),
\end{aligned}
\end{equation}
and summing,
\begin{equation}
\begin{aligned}
  \mathrm{Var}(V)
    &= p_f p_c (1 - p_c) + p_c^{2} p_f (1 - p_f) \\
    &= p_f p_c \bigl[(1 - p_c) + p_c (1 - p_f)\bigr] \\
    &= p_f p_c (1 - p_f p_c),
\end{aligned}
\end{equation}
which matches Eq.~\eqref{eq:variance_decomp}. The within-format reward-variance share therefore is
\begin{equation}
\begin{aligned}
  \Phi_{\mathrm{wf}}(p_f, p_c)
    &= \frac{\mathbb{E}[\mathrm{Var}(V \mid V_f)]}{\mathrm{Var}(V)} \\
    &= \frac{p_f p_c (1 - p_c)}{p_f p_c (1 - p_f p_c)}
     = \frac{1 - p_c}{1 - p_f p_c}.
\end{aligned}
\end{equation}
Differentiating with respect to $p_f$ at fixed $p_c \in (0, 1)$,
\begin{equation}
\begin{aligned}
  \frac{\partial}{\partial p_f}\frac{1 - p_c}{1 - p_f p_c}
    &= \frac{(1 - p_c) p_c}{(1 - p_f p_c)^{2}} \;>\; 0,
\end{aligned}
\end{equation}
so $\Phi_{\mathrm{wf}}$ is strictly increasing in $p_f$ on $(0,1)$. Its continuous
extension satisfies $\lim_{p_f\uparrow1}\Phi_{\mathrm{wf}}(p_f,p_c)=1$ and
$\lim_{p_f\downarrow0}\Phi_{\mathrm{wf}}(p_f,p_c)=1-p_c$. The original ratio is not
defined at $p_f=0$ because then $\mathrm{Var}(V)=0$. \qed

\subsection{Proof of the Positive-Correlation Tightening}
\label{sec:appendix_snr_corr}

Drop the independence assumption and reparametrize by the conditional success
rates $p := \Pr[V_c = 1 \mid V_f = 1]$ and
$q := \Pr[V_c = 1 \mid V_f = 0]$, with
$p_f\in(0,1)$ and feasible $p,q\in[0,1]$. The marginal is
$p_c = p_f \, p + (1 - p_f) \, q$. Independence corresponds to
$p = q = p_c$.

Since $V = V_f V_c$, the value of $V_c$ when $V_f = 0$ is irrelevant: $V \equiv 0$
on that event. The conditional moments become
\begin{equation}
\begin{aligned}
  \mathbb{E}[V \mid V_f = 1] &= p, \\
  \mathbb{E}[V \mid V_f = 0] &= 0, \\
  \mathrm{Var}(V \mid V_f = 1) &= p(1 - p), \\
  \mathrm{Var}(V \mid V_f = 0) &= 0,
\end{aligned}
\end{equation}
and exactly the same algebra as in Section~\ref{sec:appendix_snr} yields
\begin{equation}
\begin{aligned}
  \mathrm{Var}(V)
    &= p_f \, p \, (1 - p_f \, p), \\
  \Phi_{\mathrm{wf}}(p_f, p)
    &= \frac{1 - p}{1 - p_f \, p}.
\end{aligned}
\end{equation}
The ratio is defined for $p>0$; the boundary $p=0$ has
$\mathrm{Var}(V)=0$ because the multiplicative verifier never emits reward
$1$. The positive-correlation comparison below has $p>p_c$ with
$p_c\in(0,1)$, so it stays away from this degenerate boundary.
Note that $q$ does not appear: the share depends only on the conditional
$p = \Pr[V_c{=}1\mid V_f{=}1]$. Differentiating,
\begin{equation}
\begin{aligned}
  \frac{\partial \Phi_{\mathrm{wf}}}{\partial p}
    &= \frac{-(1 - p_f \, p) + p_f(1 - p)}
            {(1 - p_f \, p)^{2}} \\
    &= \frac{p_f - 1}{(1 - p_f \, p)^{2}}
      \;<\; 0
      \quad \text{for } p_f \in (0, 1).
\end{aligned}
\end{equation}
Holding the marginal $p_c$ fixed and staying within the feasible range
$q\in[0,1]$, the constraint
$p_c = p_f p + (1 - p_f) q$ requires $\partial p / \partial q = -(1-p_f)/p_f$,
so increasing the correlation between $V_f$ and $V_c$ (i.e., increasing $p$ at
the expense of $q$ with marginal $p_c$ fixed) strictly decreases
$\Phi_{\mathrm{wf}}$. Independence corresponds to $p=p_c$. Thus positive correlation ($p>p_c$) gives a strictly smaller share than the independent formula, while
nondegenerate negative correlation ($0<p<p_c$) reverses the inequality. The
independent formula is an upper envelope only over the
nonnegative/positive-correlation regime.
\qed

\subsection{Quantitative non-identifiability (Proposition~\ref{prop:nonid_main})}
\label{sec:appendix_nonid_proof}

We use the family $\mathcal{T}$ of semantics-preserving transformations and the
reward functional $\mathcal{R}(\pi,\rho;\bm{x})$ of \mainsec{2.3};
for $\tau(\bm{x})=(\bm{v},\bm{q},\mathcal{I}_\tau)$ write
$\mathcal{F}_\tau:=\mathcal{F}(\mathcal{I}_\tau)$.

\begin{assumption}[Restricted prompt coverage]
  \label{ass:rc}
  There are measurable sets
  $\emptyset\ne S_{\text{train}}\subseteq\mathcal{T}_{\text{train}}
  \subsetneq \mathcal{T}$ with
  $\mathcal{T}_{\text{train}},S_{\text{train}}\in\Sigma_{\mathcal{T}}$ and
  $P_{\text{train}}(S_{\text{train}})=1$.
\end{assumption}

\begin{assumption}[Disjoint unseen format]
  \label{ass:separated_format}
  There exists $\tau_B \in \mathcal{T}\setminus\mathcal{T}_{\text{train}}$
  whose verifier format constraint $\mathcal{F}_B:=\mathcal{F}_{\tau_B}$ is
  disjoint from
  $\mathcal{F}_{\text{train}}:=\bigcup_{\tau\in\mathcal{T}_{\text{train}}}
  \mathcal{F}_\tau$.
\end{assumption}

\begin{assumption}[Policy-class extension realizability]
  \label{ass:format_realizable}
  The policy class contains two
  measurable deterministic branches $\pi_1,\pi_2$ that coincide on every
  $\tau\in\mathcal{T}\setminus\{\tau_B\}$. On
  $S_{\text{train}}$ they emit the correct answer $a^*$ in the
  requested format. At $\tau_B$, $\pi_1$ emits $a^*$ in a format requested by
  some transformation in $S_{\text{train}}$, while $\pi_2$ emits $a^*$ in
  $\mathcal{F}_B$.
\end{assumption}

\begin{proposition}[Quantitative non-identifiability]
  \label{prop:nonid_main}
  Under Assumptions~\ref{ass:rc},~\ref{ass:separated_format},
  and~\ref{ass:format_realizable}, with a $\tau$-dependent verifier and reward
  $R \in [R_{\min}, R_{\max}]$ such that a correct response in
  $\mathcal{F}_\tau$ receives $R_{\max}$ and any response lying in a format
  constraint disjoint from $\mathcal{F}_\tau$ receives $R_{\min}$, the two
  branches $\pi_1 \neq \pi_2$ from Assumption~\ref{ass:format_realizable} are
  $P_{\text{train}}$-optimal with equal training reward, yet for the separated $\tau_B$ and
  any $P_{\text{test}}\in\Delta(\mathcal{T})$ with atomic mass
  $m_B := P_{\text{test}}(\{\tau_B\}) > 0$, where
  $\Delta_R := R_{\max}-R_{\min}$,
  \begin{equation}
  \begin{aligned}
    &\mathcal{R}(\pi_2,P_{\text{test}};\bm{x})
    - \mathcal{R}(\pi_1,P_{\text{test}};\bm{x}) \ge m_B\Delta_R.
    \label{eq:test_gap}
  \end{aligned}
  \end{equation}
  Hence the $P_{\text{train}}$-optimum set has diameter $\ge m_B\Delta_R$ under
  $P_{\text{test}}$-reward; this is a population-level identifiability statement
  about the objective.
\end{proposition}

\noindent\textit{Proof.}
By Assumption~\ref{ass:separated_format}, take a
$\tau_B \in \mathcal{T}\setminus\mathcal{T}_{\text{train}}$ whose formatting
instruction induces a constraint $\mathcal{F}_B$ disjoint from
$\mathcal{F}_{\text{train}}$ (the union of formats demanded by
$\tau \in \mathcal{T}_{\text{train}}$). Assumption~\ref{ass:format_realizable}
places the two extensions used below in the policy class.

\textit{Construction.}
Use the two branches supplied by Assumption~\ref{ass:format_realizable}. They
coincide on $S_{\text{train}}$ and emit a correct semantic
answer in the requested format there. Choose the training-set format used
by $\pi_1$ at $\tau_B$ and call it
$\mathcal{F}_0=\mathcal{F}_{\tau_0}\subseteq\mathcal{F}_{\text{train}}$ for
some $\tau_0\in S_{\text{train}}$; at $\tau_B$, $\pi_1$ emits a
correct semantic answer matching $\mathcal{F}_0$, while $\pi_2$ emits a correct
semantic answer matching $\mathcal{F}_B$. This makes $\pi_1\ne\pi_2$ while
keeping the two policies identical on $S_{\text{train}}$.

\textit{Training-reward equality.}
On the $P_{\text{train}}$-full set $S_{\text{train}}$, the two policies are identical by
construction and emit a correct response in the requested format. The endpoint
reward condition gives reward $R_{\max}$ on each transformation in
$S_{\text{train}}$, so both policies are $P_{\text{train}}$-optimal and
  \begin{equation}
  \mathcal{R}(\pi_1,P_{\text{train}};\bm{x})
    \;=\;
  \mathcal{R}(\pi_2,P_{\text{train}};\bm{x}),
  \label{eq:train_match}
  \end{equation}
proving Eq.~\eqref{eq:train_match}.

\textit{Test-time gap.}
For any $P_{\text{test}} \in \Delta(\mathcal{T})$ assigning atomic mass
$m_B = P_{\text{test}}(\{\tau_B\}) > 0$ to $\tau_B$, the $\tau_B$-verifier requires
$\mathcal{F}_B$. On $\tau_B$, $\pi_1$ receives $R_{\min}$ and $\pi_2$ receives
$R_{\max}$. On every other transformation the two policies coincide, so the
pointwise difference
$\mathcal{R}(\pi_2,\tau;\bm{x})-\mathcal{R}(\pi_1,\tau;\bm{x})$ is
zero and is exactly $\Delta_R$ at $\tau_B$. Therefore
\begin{equation}
\begin{aligned}
  &\mathcal{R}(\pi_2,P_{\text{test}};\bm{x})
    - \mathcal{R}(\pi_1,P_{\text{test}};\bm{x})
    \\
  &\quad\ge m_B \cdot (R_{\max} - R_{\min}) \\
  &\quad= m_B \Delta_R,
\end{aligned}
\end{equation}
proving Eq.~\eqref{eq:test_gap}. The bound is tight on this construction.

\textit{Diameter of the $P_{\text{train}}$-optimum set.}
The training objective constrains behavior only on the
$P_{\text{train}}$-full training set; behavior on the unseen $\tau_B$ is
unconstrained by the training objective. Since both policies attain
$R_{\max}$ on $S_{\text{train}}$, both are $P_{\text{train}}$-optimal; their
$P_{\text{test}}$-rewards differ by
$m_B\Delta_R$. Hence the diameter of the $P_{\text{train}}$-optimum set under
$P_{\text{test}}$-reward is at least $m_B\Delta_R$, and \emph{some}
$P_{\text{train}}$-optimal policy (specifically $\pi_1$) has a
$P_{\text{test}}$ comparator gap $\ge m_B\Delta_R$ relative to $\pi_2$. We do
not claim every $P_{\text{train}}$-optimal policy has such a gap, since
$\pi_2$ itself is $P_{\text{train}}$-optimal.
The proposition is a population-level identifiability statement about the
objective, not a learnability claim about any particular training algorithm.
\qed

\subsection{Measure-zero training support and low-complexity shortcuts}
\label{sec:appendix_measure_zero}

The prompt-transformation construction above is a finite-support instance of a broader identifiability obstruction.

\begin{proposition}[Measure-zero training support]\label{prop:measure_zero}
Let $\mathcal{W}\subseteq\mathbb{R}^D$ be measurable and let the deployment
distribution $P_W$ satisfy $P_W\ll\lambda_D$, where $\lambda_D$ is
$D$-dimensional Lebesgue measure. Let
$\Omega\subseteq\mathcal{W}$ be measurable, and let the training distribution
$P_{\mathrm{tr}}$ satisfy $P_{\mathrm{tr}}(\Omega)=1$, where $\Omega$ is finite or
contained in a measurable $d$-dimensional embedded submanifold with $d<D$.
Let the answer space
$\mathcal{Y}_{\mathrm{ans}}$ carry the discrete sigma-algebra and contain at
least two distinct labels, and let the deployment semantic rule
$A^\star:\mathcal{W}\to\mathcal{Y}_{\mathrm{ans}}$ be measurable. For any
measurable bounded reward $R$ and training objective
\begin{equation}
  J(\pi)=
  \mathbb{E}_{x\sim P_{\mathrm{tr}}}
  \mathbb{E}_{y\sim\pi(\cdot\mid x)}[R(x,y)]
\end{equation}
evaluated only under $P_{\mathrm{tr}}$, any two policies that agree on
$\Omega$ in their conditional output laws have the same objective value.
Consequently, without additional structural restrictions on admissible rules,
there are infinitely many measurable shortcut rules
$a:\mathcal{W}\to\mathcal{Y}_{\mathrm{ans}}$ such that $a=A^\star$ on $\Omega$
but
$P_W(\{x:a(x)\ne A^\star(x)\})>0$, and each deterministic policy
$\pi_a(\cdot\mid x)=\delta_{a(x)}$ is training-equivalent to
$\pi_{A^\star}$. If training-optimal policies are then selected by the
tie-breaker $\arg\min_{\pi\in\Pi^\star}C(\pi)$, where
$\Pi^\star=\{\pi:J(\pi)=J^\star\}$, and if
$\pi_{A^\star}\in\Pi^\star$, then every training-equivalent shortcut also lies
in $\Pi^\star$; any such shortcut with
$C(\pi_a)<C(\pi_{A^\star})$ is preferred to the semantic policy by that
tie-breaker.
\end{proposition}

\begin{proof}
Finite sets and lower-dimensional embedded submanifolds have
$\lambda_D$-measure zero, so $P_W(\Omega)=0$ by absolute continuity. Since
$P_{\mathrm{tr}}(\Omega)=1$, the objective $J$ only depends on a policy's
conditional output law on $\Omega$; policies agreeing there are therefore
identical under $J$.

It remains to show that many disagreeing extensions exist. Choose two distinct
labels $b_0,b_1\in\mathcal{Y}_{\mathrm{ans}}$. Because $P_W\ll\lambda_D$ is
non-atomic and $P_W(\mathcal{W}\setminus\Omega)=1$, there are infinitely many
measurable sets $B\subseteq\mathcal{W}\setminus\Omega$ with $P_W(B)>0$.
For each such $B$, define
\begin{equation}
  a_B(x)=
  \begin{cases}
    A^\star(x), & x\notin B,\\
    b_1, & x\in B \text{ and } A^\star(x)=b_0,\\
    b_0, & x\in B \text{ and } A^\star(x)\ne b_0 .
  \end{cases}
\end{equation}
Then $a_B=A^\star$ on $\Omega$ and $a_B\ne A^\star$ on $B$, so the deployment
disagreement has positive $P_W$-mass while $J(\pi_{a_B})=J(\pi_{A^\star})$.
Finally, if $\pi_{A^\star}\in\Pi^\star$, then under the stated
minimum-complexity tie-breaker over $\Pi^\star$, any training-equivalent
shortcut with smaller $C$ than $\pi_{A^\star}$ prevents
$\pi_{A^\star}$ from being selected. This is an
identifiability and tie-breaking statement, not a claim that arbitrary
optimization dynamics must find such a shortcut.
\end{proof}

\subsection{Dense Rewards Without Invariant Semantic Coverage}
\label{sec:appendix_dense_proof}

This extension of the non-identifiability construction
(\mainsec{2.3}) is orthogonal to the main paper, which uses the
trinary reward of \mainsec{3}; we include it for completeness.

\begin{proposition}[Dense rewards without invariant semantic coverage]\label{prop:dense}
Suppose the dense reward can be decomposed, after absorbing component weights,
as
\begin{equation}
  R(\bm{y},\tau(\bm{x})) =
  R_{\mathrm{dep}}(\bm{y},\tau(\bm{x}))
  + R_{\mathrm{inv}}(\bm{y},\bm{x}),
\end{equation}
where $R_{\mathrm{dep}}$ contains the prompt-dependent verifier or formatting
components and $R_{\mathrm{inv}}$ is prompt-invariant. For
$j\in\{\mathrm{dep},\mathrm{inv}\}$, write
$\mathcal{R}_{j}(\pi,\tau;\bm{x})$ for the expected value of $R_j$ under
$\bm{y}\sim\pi(\cdot\mid\tau(\bm{x}))$. Assume the components are measurable
and integrable under the policies considered. Let $\Pi$ be the admissible
policy class. Define
\begin{equation}
\begin{aligned}
  J_{\mathrm{dense}}(\pi):=
  &\mathbb{E}_{\tau \sim P_{\text{train}}}
  \mathbb{E}_{\bm{y} \sim \pi(\cdot \mid \tau(\bm{x}))} 
  \!\left[
      R_{\mathrm{dep}}(\bm{y},\tau(\bm{x}))
      + R_{\mathrm{inv}}(\bm{y},\bm{x})
    \right].
\end{aligned}
\end{equation}
Assume there is a measurable policy kernel
$\bar\pi\in\arg\max_{\pi\in\Pi}J_{\mathrm{dense}}(\pi)$. Assume $\Pi$ can
extend $\bar\pi$ to two
measurable policies that agree $P_{\text{train}}$-a.s., agree away from
$\tau_B$, and use the two unseen $\tau_B$ branches from
Proposition~\ref{prop:nonid_main}. For the resulting extensions
$\pi_1,\pi_2$, define the unseen-format component gaps
\begin{equation}
\begin{aligned}
  \Delta_{\mathrm{dep}}
    &:=
    \mathcal{R}_{\mathrm{dep}}(\pi_2,\tau_B;\bm{x})
    - \mathcal{R}_{\mathrm{dep}}(\pi_1,\tau_B;\bm{x}), \\
  \Delta_{\mathrm{inv}}
    &:=
    \mathcal{R}_{\mathrm{inv}}(\pi_1,\tau_B;\bm{x})
    - \mathcal{R}_{\mathrm{inv}}(\pi_2,\tau_B;\bm{x}).
\end{aligned}
\end{equation}
Then $\pi_1$ and $\pi_2$ are both $P_{\text{train}}$-optimal, the training
equality in Eq.~\eqref{eq:train_match} still holds, and for any
$P_{\text{test}}\in\Delta(\mathcal{T})$ with
$m_B=P_{\text{test}}(\{\tau_B\})>0$,
\begin{equation}
\begin{aligned}
  &\mathcal{R}(\pi_2,P_{\text{test}};\bm{x})
   - \mathcal{R}(\pi_1,P_{\text{test}};\bm{x}) = m_B(\Delta_{\mathrm{dep}}-\Delta_{\mathrm{inv}}).
\end{aligned}
\end{equation}
Thus the dense-reward construction yields a positive test-time comparator gap
between equally $P_{\text{train}}$-optimal policies whenever
$\Delta_{\mathrm{dep}}>\Delta_{\mathrm{inv}}$.
If the argmax assumption is replaced by a nonattained supremum, the same
construction with any
$\bar\pi_\zeta\in\Pi$ satisfying
$J_{\mathrm{dense}}(\bar\pi_\zeta)\ge \sup_{\pi\in\Pi}J_{\mathrm{dense}}(\pi)-\zeta$
and admitting the same extensions gives two equally $\zeta$-optimal extensions
with the same training equality and test-gap identity.
\end{proposition}

\noindent\textit{Proof.}
For the fixed task instance $\bm{x}$, consider the dense training objective
\begin{equation}
\begin{aligned}
  \max_\theta\;
  &\mathbb{E}_{\tau \sim P_{\text{train}}}
   \mathbb{E}_{\bm{y} \sim \pi_\theta(\cdot \mid \tau(\bm{x}))} \bigl[
      R_{\mathrm{dep}}(\bm{y},\tau(\bm{x}))
      +R_{\mathrm{inv}}(\bm{y},\bm{x})
    \bigr].
\end{aligned}
\end{equation}
This fixed-$\bm{x}$ objective is still evaluated only under
$\tau \sim P_{\text{train}}$.
Choose the dense-objective optimum $\bar\pi$ from the assumption and extend it
in two ways: $\pi_1$ and $\pi_2$ both equal $\bar\pi$
$P_{\text{train}}$-a.s. and share any behavior outside that full-measure set
except at $\tau_B$; at $\tau_B$, $\pi_1$ emits the fixed training-time format
$\mathcal{F}_0$ and $\pi_2$ emits the requested format $\mathcal{F}_B$, as in
Proposition~\ref{prop:nonid_main}. Because the training objective is
insensitive to changes outside a $P_{\text{train}}$-full set, both extensions
have the same dense training value as $\bar\pi$ and are therefore
$P_{\text{train}}$-optimal. Component-wise equality on that full-measure set
also gives Eq.~\eqref{eq:train_match}.
If only a $\zeta$-optimal branch is available, the same off-support extension
argument preserves that $\zeta$-optimality and all displayed equalities below.

On every transformation other than $\tau_B$ the two policies also coincide in
the construction, so the dense test-reward difference is concentrated on
$\tau_B$. At $\tau_B$, the prompt-dependent contribution to
$\mathcal{R}(\pi_2,\tau_B;\bm{x})-\mathcal{R}(\pi_1,\tau_B;\bm{x})$ is
$\Delta_{\mathrm{dep}}$, while the prompt-invariant contribution is
$-\Delta_{\mathrm{inv}}$ by definition. Multiplying by the test mass
$m_B=P_{\text{test}}(\{\tau_B\})$ gives the displayed identity.

If the prompt-invariant component assigns equal value to the two semantically
correct answers, then $\Delta_{\mathrm{inv}}=0$ and the positive gap is exactly
$m_B\Delta_{\mathrm{dep}}$. More generally, prompt-invariant components only
remove this particular comparator-gap construction when their advantage for
$\pi_1$ is at least the prompt-dependent advantage of $\pi_2$, i.e.,
$\Delta_{\mathrm{inv}}\ge\Delta_{\mathrm{dep}}$.
\qed

\subsection{DTR reward-variation floor}
\label{sec:appendix_dtr_proof}

\begin{theorem}[$p_f$-independent reward-side share under DTR]
  \label{thm:dtr_decouples}
  Let $R = R_t + \alpha_{\mathrm{rew}} R_s$ be the DTR reward of \maineq{6}, with
  $R_t\in\{-1,0,+1\}$ assigned by \maineq{5},
  $R_s\in[0,1]$, and fixed $\alpha_{\mathrm{rew}} > 0$. Write
  $p_f:=\Pr[V_f=1]\in(0,1)$.
  Let $V_f$ be the requested-format indicator and assume the extractor is
  format-exact on this prompt:
  $\mathcal{A}_{\tau}(\bm{y})\ne\emptyset$ iff $V_f=1$.
  Define the extractor-induced correctness indicator
  $V_c^{\mathrm{ext}}=\mathbf{1}\{\mathcal{A}_{\tau}(\bm{y})=a^*\}$. Then
  \maineq{5} gives $R_t=-1$ when $V_f=0$ and
  $R_t=V_c^{\mathrm{ext}}$ when $V_f=1$. Assume $R_s$ has
  non-degenerate within-bucket dispersion: for some feasible constant
  $m\in(0,1/4]$,
  independent of the $p_f$ values in the analyzed rollout-law class,
  $\mathrm{Var}(R_s \mid V_c^{\mathrm{ext}}{=}v,V_f{=}v')\ge m$ in every
  extractor-induced joint bucket $(v,v')$ with positive probability.
  Then
  $\mathrm{Var}(R \mid V_f{=}v') \ge \alpha_{\mathrm{rew}}^{2} m$ for each positive-probability
  format bucket $v' \in \{0,1\}$, so
  for
  $\Phi_{\mathrm{wf}}^{\mathrm{DTR}}
  :=\mathbb{E}[\mathrm{Var}(R\mid V_f)]/\mathrm{Var}(R)$,
  \begin{equation}
    \Phi_{\mathrm{wf}}^{\mathrm{DTR}}
    \ge \frac{4\alpha_{\mathrm{rew}}^2 m}{(\alpha_{\mathrm{rew}}+2)^2}.
  \end{equation}
  This lower bound does not depend on $p_f$.
\end{theorem}

\begin{proof}
Recall $R = R_t + \alpha_{\mathrm{rew}} R_s$ with $R_t \in \{-1, 0, +1\}$ assigned by
\maineq{5}, $R_s \in [0, 1]$, and
$\alpha_{\mathrm{rew}} > 0$ a fixed reward weight. Write
$\sigma_{v, v'}^{2} := \mathrm{Var}(R_s \mid V_c^{\mathrm{ext}} = v, V_f = v')$
for joint buckets with positive probability and
$m \le \sigma_{v, v'}^{2}$ over those buckets, with the same $m\in(0,1/4]$ for the
rollout distributions in the analyzed class (the assumption of
Theorem~\ref{thm:dtr_decouples}: $R_s$ is non-degenerate within every
positive-probability joint bucket). We take $p_f\in(0,1)$ so both format
buckets have well-defined
conditional variances.

\textit{Conditional variance under $V_f = 0$.}
When $V_f = 0$, $R_t = -1$ is constant and $R = -1 + \alpha_{\mathrm{rew}} R_s$. By the law
of total variance applied conditionally on $V_f = 0$,
\begin{equation}
\begin{aligned}
  \mathrm{Var}(R \mid V_f = 0)
    &= \alpha_{\mathrm{rew}}^{2} \mathrm{Var}(R_s \mid V_f = 0) \\
    &\ge \alpha_{\mathrm{rew}}^{2}\,
       \mathbb{E}\!\left[\sigma_{V_c^{\mathrm{ext}},0}^{2}\mid V_f=0\right] \\
    &\ge \alpha_{\mathrm{rew}}^{2} m,
\end{aligned}
\end{equation}
where the second inequality drops the (non-negative) between-$V_c^{\mathrm{ext}}$ term and
uses $\sigma_{V_c^{\mathrm{ext}}, 0}^{2} \ge m$.

\textit{Conditional variance under $V_f = 1$.}
When $V_f = 1$, $R_t = V_c^{\mathrm{ext}}$ and
$R = V_c^{\mathrm{ext}} + \alpha_{\mathrm{rew}} R_s$. By the same total-variance
step,
\begin{equation}
\begin{aligned}
  \mathrm{Var}(R \mid V_f = 1)
    &\ge \mathbb{E}\!\left[\alpha_{\mathrm{rew}}^{2}\sigma_{V_c^{\mathrm{ext}},1}^{2}\mid V_f=1\right] \\
    &\ge \alpha_{\mathrm{rew}}^{2} m.
\end{aligned}
\end{equation}

\textit{Lower bound on $\mathbb{E}[\mathrm{Var}(R \mid V_f)]$.}
\begin{equation}
\begin{aligned}
  \mathbb{E}\bigl[\mathrm{Var}(R \mid V_f)\bigr]
    &= p_f\,\mathrm{Var}(R \mid V_f{=}1) + (1{-}p_f)\,\mathrm{Var}(R \mid V_f{=}0) \\
    &\ge \alpha_{\mathrm{rew}}^{2} m \quad \text{for all } p_f \in (0,1).
\end{aligned}
\end{equation}
The bound is independent of $p_f$.

\textit{Upper bound on $\mathrm{Var}(R)$.}
With $\alpha_{\mathrm{rew}}$ fixed and $R_s\in[0,1]$,
$R \in [-1, \alpha_{\mathrm{rew}} + 1]$, so Popoviciu's
inequality gives $\mathrm{Var}(R) \le (\alpha_{\mathrm{rew}}+2)^{2}/4$.

\textit{Lower bound on $\Phi_{\mathrm{wf}}^{\mathrm{DTR}}$.}
\begin{equation}
\begin{aligned}
  \Phi_{\mathrm{wf}}^{\mathrm{DTR}}
    &= \frac{\mathbb{E}[\mathrm{Var}(R \mid V_f)]}{\mathrm{Var}(R)} \\
    &\ge \frac{4\alpha_{\mathrm{rew}}^{2} m}{(\alpha_{\mathrm{rew}}+2)^{2}}
    \quad \text{for all } p_f \in (0, 1).
\end{aligned}
\end{equation}
The bound depends only on $\alpha_{\mathrm{rew}}$ and $m$, not on $p_f$, so the
C1 variance-floor requirement is satisfied. \emph{Remark.} Popoviciu's
inequality is tight only for two-point
distributions on the endpoints, so the closed-form constant can be loose. We
report it for concreteness; the qualitative content (the $p_f$-independence of
the reward-side share floor) is the load-bearing claim, not the numerical value of the
constant.
\end{proof}

\subsection{Embedding-space adversary as local ascent on a relaxed inner objective}
\label{sec:appendix_adv_proof}

\begin{proposition}[Local ascent on the relaxed adversarial objective]\label{prop:adv_inf}
For a fixed policy, Eq.~\eqref{eq:adv_obj} is an embedding-space relaxation of
the negative inner objective of \maineq{4}, restricted to the
$\epsilon_{\mathrm{emb}}$-ball
$\{\bm{\delta}:\|\bm{\delta}\|_2\le\epsilon_{\mathrm{emb}}\}$. When the relaxed
objective is smooth (or when only the stop-gradient score-function component is
optimized), projected ascent gives a local first-order ascent method for that
restricted objective/component; the stochastic implementation should be read as a
local search heuristic, not as a per-step monotonicity guarantee or a claim of
reaching the global inner infimum of \maineq{4}.
\end{proposition}

\begin{proof}
For a deterministic perturbation, or after fixing a single reparameterized
sample from $q_\phi$, Eq.~\eqref{eq:adv_obj} reduces locally to maximizing
\begin{equation}
\begin{aligned}
  \mathcal{J}_\phi(\bm{\delta})
  &:= -\,
    \mathbb{E}_{\bm{y} \sim \pi_\theta(\cdot \mid \tau_\phi(\bm{x}))}\!
    \bigl[R_{\mathrm{DTR}}(\bm{y}, \tau_\phi(\bm{x}))\bigr], \ \ \|\bm{\delta}\|_2 \le \epsilon_{\mathrm{emb}}.
\end{aligned}
\end{equation}
This is the negative relaxed inner objective over the
$\epsilon_{\mathrm{emb}}$-ball, where $\tau_\phi(\bm{x})$ is the prompt with
perturbed instruction embedding
$\bm{e}_{\mathcal{I}}(\bm{x})+\bm{\delta}$. For a stochastic $q_\phi$, the
objective is the expectation of this quantity over the sampled perturbation.
When the reward path through $\tau_\phi$ is treated as a stop-gradient, the
REINFORCE estimator (Eq.~\eqref{eq:ggds_reinforce}) provides a stochastic
estimate of the
score-function component of $\nabla_{\bm{\delta}}\mathcal{J}_\phi$; if the
reward is differentiable through $\tau_\phi$, the pathwise term described in
Section~\ref{sec:appendix_adv} must be added for the full gradient. For exact gradients
of an $L$-smooth relaxed objective, projected gradient ascent with sufficiently
small step size has the usual first-order local ascent property unless the
iterate is already first-order stationary for the constrained problem. With
stochastic rollouts, the update inherits this interpretation only under standard
unbiased-or-positive-correlation and bounded-variance assumptions, and need not
increase $\mathcal{J}_\phi$ on every realized step. This is the approximation
statement; it does not imply global minimax convergence or attainment of the
unrestricted $\Delta(\mathcal{T})$ infimum.
\end{proof}

\subsection{Population-level consequences of C1/C2 (Proposition~\ref{prop:c1c2_consequences})}
\label{sec:appendix_c1c2_consequences}

\begin{proposition}[Population-level consequences of C1/C2]
\label{prop:c1c2_consequences}
Under $\tau$-dependent extraction $\mathcal{A}(\bm{y}; d_\tau)$ and bounded
reward $R\in[R_{\min},R_{\max}]$, the following claims hold. \emph{(i)} (C1)
bounds $\Phi_{\mathrm{wf}}^R$ below independently of $p_f$. \emph{(ii)} If the coverage
$\rho_{\bm{x}}$ places masses $w_{\mathrm{tr}},w_B>0$ on some
$\tau_{\text{train}}\in S_{\text{train}}$ from
Assumption~\ref{ass:rc} and on the separated
$\tau_B$, then the
format-fixed branch $\pi_1$ of Prop.~\ref{prop:nonid_main} violates (C2)
whenever $\varepsilon < w_{\mathrm{tr}}w_B/\kappa_D^2$: its answer-marginal
collapses to $\delta_{a^*}$ under $\tau_{\text{train}}$ and to
$\delta_{\emptyset}$ under $\tau_B$. If the deterministic correct-answer branch
$\pi_2$ from Prop.~\ref{prop:nonid_main} is realizable on the finite atom
support $\{\tau\in\mathcal{C}(\bm{x}):\rho_{\bm{x}}(\{\tau\})>0\}$ of
$\rho_{\bm{x}}$, it extracts $a^*$ throughout that atom support and is
(C2)-feasible. \emph{(iii)} (C1) supplies non-degenerate
within-format reward variation; by itself it does not guarantee that the
semantic score is aligned with correctness. Selecting the high-reward
instruction-following branch additionally requires a semantically aligned
$R_s$ and population optimization over the C2-feasible class. Algorithmic
convergence is a separate empirical question. \emph{(iv)} For finite
$\varepsilon$, if the reward is a common bounded function of the
prompt-conditioned extracted answer used in (C2), the $\rho_{\bm{x}}$-averaged
absolute pairwise expected reward difference is at most
$\Delta_R\kappa_D\sqrt{\varepsilon}$, where
$\Delta_R:=R_{\max}-R_{\min}$; if a
semantic scorer reads the full response or changes with the prompt, the same
bound requires C2 on that score-relevant marginal or an additional term for
reward-function shift.
\end{proposition}

\noindent\textit{Proof.}
\textbf{Part (i): (C1) implies bounded $\Phi_{\mathrm{wf}}^R$.}
Definition~\ref{def:rd} states that under (C1),
each positive-probability format bucket has conditional variance at least $c$,
hence
$\mathbb{E}[\mathrm{Var}(R \mid V_f)] \ge c$. The total $\mathrm{Var}(R)$ is
upper-bounded by $(R_{\max} - R_{\min})^{2} / 4$ (bounded range). Hence
\begin{equation}
\begin{aligned}
  \Phi_{\mathrm{wf}}^{R}
    &= \frac{\mathbb{E}[\mathrm{Var}(R \mid V_f)]}{\mathrm{Var}(R)} \\
    &\ge \frac{4 c}{(R_{\max} - R_{\min})^{2}}
      \;>\; 0,
\end{aligned}
\end{equation}
independently of $p_f$; the analogous reward-variance floor removes the
$p_f$-dependence present in Theorem~\ref{thm:snr}.

\textbf{Part (ii): (C2) blocks the prompt-dependent branch.}
Definition~\ref{def:ie} fixes a non-degenerate coverage distribution
$\rho_{\bm{x}} \in \Delta(\mathcal{C}(\bm{x}))$ and requires
\begin{equation}
  \mathbb{E}_{\tau_1, \tau_2 \sim \rho_{\bm{x}}}
  \!\left[D\!\left(P_{\tau_1}^{\theta}\,\middle\|\,P_{\tau_2}^{\theta}\right)\right]
  \;\le\; \varepsilon,
\end{equation}
where $P_{\tau}^{\theta}:=
\pi^{\mathcal{A}_{\tau}}_\theta(\cdot\mid\tau(\bm{x}))$ is the answer marginal
under the
$\tau$-dependent extraction $\mathcal{A}(\bm{y}; d_\tau)$ of
\maineq{5}: a rollout whose surface format matches $\tau$
extracts a valid answer, otherwise the extractor returns $\emptyset$.
Assume $\rho_{\bm{x}}$ places masses
$w_{\mathrm{tr}}:=\rho_{\bm{x}}(\tau_{\text{train}})>0$ on some
$\tau_{\text{train}}\in S_{\text{train}}$ and
$w_B:=\rho_{\bm{x}}(\tau_B)>0$ on the separated $\tau_B$ from
Assumption~\ref{ass:separated_format}. As $\varepsilon\to0$, every prompt pair
in this support with positive mass must induce matching answer-marginals.

For the training transformation $\tau_{\text{train}}$ and separated
$\tau_B$ in the construction, $\pi_1$ emits a training-compatible formatted
string in both cases. Under $\tau_{\text{train}}$, the extractor returns the
correct answer $a^*$; under $\tau_B$, the emitted format lies in
$\mathcal{F}_0\subseteq\mathcal{F}_{\text{train}}$, which is disjoint from
$\mathcal{F}_B$, so extractor consistency gives $\emptyset$. The
answer-marginal of $\pi_1$ is therefore
$\delta_{a^*}$ under $\tau_{\text{train}}$ and $\delta_{\emptyset}$ under
$\tau_B$, two atoms on disjoint support of
$\mathcal{Y}_{\text{ans}}\cup\{\emptyset\}$. Their total variation distance is
$1$, so any divergence satisfying the Pinsker-type condition in
Definition~\ref{def:ie} must obey
$D(\delta_{a^*}\|\delta_{\emptyset})\ge 1/\kappa_D^2$; for KL the divergence is
infinite. Consequently the left-hand side of (C2) is at least
$w_{\mathrm{tr}}w_B/\kappa_D^2$ (conservatively using one ordered pair), and
$\pi_1$ \emph{violates} (C2) whenever
$\varepsilon < w_{\mathrm{tr}}w_B/\kappa_D^2$. In particular, it is excluded in
the $\varepsilon\to0$ regime.

If the deterministic correct-answer branch $\pi_2$ is realizable on
$\{\tau\in\mathcal{C}(\bm{x}):\rho_{\bm{x}}(\{\tau\})>0\}$, then $\pi_2$ emits a correctly formatted
string under each $\tau$ in that support, so every prompt-conditioned extractor
returns the same $a^*$. Thus for every ordered pair
$(\tau_1,\tau_2)$ in the support of $\rho_{\bm{x}}$, both answer marginals are
$\delta_{a^*}$ and
\begin{equation}
  D\!\left(
    \pi_2^{\mathcal{A}_{\tau_1}}(\cdot\mid\tau_1(\bm{x}))
    \,\middle\|\,
    \pi_2^{\mathcal{A}_{\tau_2}}(\cdot\mid\tau_2(\bm{x}))
  \right)=0.
\end{equation}
Therefore the (C2) expectation is zero and $\pi_2$ \emph{satisfies} (C2) at
every $\varepsilon \ge 0$.

Thus (C2) at the answer level admits the prompt-following branch $\pi_2$ and
excludes the format-fixed branch $\pi_1$ in the $\varepsilon \to 0$ regime,
the exact direction needed to remove this non-identifiable branch from
Prop.~\ref{prop:nonid_main}. (C1) provides non-degenerate within-format reward
variation. It does not, by itself, prove that the gradient prefers semantic
correctness; that stronger optimization claim additionally requires the
semantic score $R_s$ to be aligned with correctness and population optimization
over the C2-feasible class.

\textit{Finite-$\varepsilon$ expected-reward-difference bound (averaged form).}
Let $\pi^\varepsilon$ denote any policy satisfying (C2) at level
$\varepsilon$ under the coverage distribution $\rho_{\bm{x}}$, and write
$\pi^{\varepsilon,\mathcal{A}_{\tau}}$ for its prompt-conditioned answer
marginal under extractor $\mathcal{A}_{\tau}$.
Definition~\ref{def:ie} restricts $D$ to an $f$-divergence satisfying a
Pinsker-type bound $\mathrm{TV}(P, Q) \le \kappa_D \sqrt{D(P \| Q)}$
(for the natural-log KL used here, $\kappa_D = 1/\sqrt{2}$). Since (C2)
is an expectation under $\rho_{\bm{x}}$, we
bound the $\rho_{\bm{x}}$-averaged pairwise expected reward difference rather than
pointwise regret. For this finite-$\varepsilon$ bound we additionally require
the evaluated reward to be a common answer-level function for the fixed task
instance: there is a bounded
$r_{\bm{x}}:\mathcal{Y}_{\text{ans}}\cup\{\emptyset\}\to[R_{\min},R_{\max}]$
such that
$R(\bm{y},\tau(\bm{x}))=r_{\bm{x}}(\mathcal{A}_\tau(\bm{y}))$ for every
$\tau$ in the coverage support. This holds for prompt-independent answer-level
exact-match and numerical-tolerance rewards after extraction. If $R_s$ reads
the full reasoning trace, or if the reward function itself changes with
$\tau$, the same argument requires C2 on the corresponding score-relevant
marginal or an additional reward-function-shift term. Under the common
answer-level reward condition, Pinsker applies exactly at the
\emph{answer-marginal} level. For compactness write
\begin{equation}
\begin{aligned}
P_i^\varepsilon
&:=\pi^{\varepsilon,\mathcal{A}_{\tau_i}}(\cdot\mid\tau_i(\bm{x})),\\
m_i^\varepsilon
&:=\mathbb{E}_{a\sim P_i^\varepsilon}[r_{\bm{x}}(a)] .
\end{aligned}
\end{equation}
\begin{equation}
\begin{aligned}
  &\mathbb{E}_{\tau_1, \tau_2 \sim \rho_{\bm{x}}}
    \bigl|m_1^\varepsilon-m_2^\varepsilon\bigr| \\
  &\quad\le\; \Delta_R \,
              \mathbb{E}_{\tau_1, \tau_2 \sim \rho_{\bm{x}}}
              \mathrm{TV}\bigl(P_1^\varepsilon,P_2^\varepsilon\bigr) \\
  &\quad\le\; \Delta_R \, \kappa_D \,
              \mathbb{E}_{\tau_1, \tau_2 \sim \rho_{\bm{x}}}\!
              \sqrt{D\bigl(P_1^\varepsilon\,\|\,P_2^\varepsilon\bigr)} \\
  &\quad\le\; \Delta_R \, \kappa_D \,
              \sqrt{\mathbb{E}_{\tau_1, \tau_2 \sim \rho_{\bm{x}}}
                    D\bigl(P_1^\varepsilon\,\|\,P_2^\varepsilon\bigr)}
              \\
  &\quad\le\; \Delta_R \, \kappa_D \, \sqrt{\varepsilon},
\end{aligned}
\end{equation}
where the third inequality is Jensen on the concave $\sqrt{\cdot}$ and the
last applies (C2). The $\rho_{\bm{x}}$-averaged answer-level mean reward
difference of $\pi^\varepsilon$ thus vanishes as $\varepsilon \to 0$. A uniform pointwise
bound requires the stronger sup-variant of (C2)
($\sup_{\tau_1,\tau_2}\!D \le \varepsilon$). \qed

\subsection{Coverage-dependent transfer bound}
\label{sec:appendix_coverage}

\begin{proposition}[Coverage-dependent transfer bound]\label{prop:coverage}
Let the evaluated reward satisfy
$R(\bm{y},\tau(\bm{x}))\in[R_{\min},R_{\max}]$ for all $\bm{y}$ and $\tau$, and
write $\Delta_R:=R_{\max}-R_{\min}$. Let
$P_{\text{test}}\in\Delta(\mathcal{T})$ be any test transformation distribution
and let $\rho_{\bm{x}}$ be any coverage distribution on
$\mathcal{C}(\bm{x})$, viewed as a probability measure on $\mathcal{T}$ by zero
extension outside $\mathcal{C}(\bm{x})$. Define the coverage gap
$\eta := \mathrm{TV}(P_{\text{test}}, \rho_{\bm{x}})$. For any policy $\pi$,
\begin{equation}
  \bigl|\mathcal{R}(\pi,P_{\text{test}};\bm{x}) -
        \mathcal{R}(\pi,\rho_{\bm{x}};\bm{x})\bigr|
  \le \Delta_R\,\eta.
  \label{eq:coverage_transfer}
\end{equation}
That is, the test-time mean reward of $\pi$ differs from its training-coverage
mean by at most $\Delta_R\,\eta$. The bound vanishes as $\rho_{\bm{x}}$
approximates $P_{\text{test}}$ in TV. If $P_{\text{test}}$ is non-atomic while
$\rho_{\bm{x}}$ is a finite zero-extended pool, this TV gap can be $1$, so the
statement is useful only when small $\eta$ is a genuine coverage assumption or
the comparison is made on the same discrete transformation pool. This
distribution-shift bound does not use (C2); (C2) matters only when the coverage
reward is used as a proxy for other prompts.

\textbf{Important scope.} This bound transfers $\pi$'s own
reward across distributions; it does \emph{not} bound the regret of
$\pi$ against an $\llbracket\bm{x}\rrbracket$-optimal $\pi^\star$.
A constant-wrong, $0$-C2-feasible policy is a counter-example
to any regret-vs-$\pi^\star$ statement at $\eta{=}0$: invariance does not
imply optimality. Closing the gap to $\pi^\star$ additionally requires
semantically aligned rewards and successful population optimization; (C1)
only requires non-degenerate within-bucket reward variation.
\end{proposition}

\begin{proof}
For any bounded $f:\mathcal{T}\to\mathbb{R}$ with
$\mathrm{osc}(f)\le\Delta_R$ and probability measures $P, Q$ on
$\mathcal{T}$,
$|\mathbb{E}_P[f] - \mathbb{E}_Q[f]| \le \mathrm{osc}(f)\,\mathrm{TV}(P, Q)
\le \Delta_R\,\mathrm{TV}(P,Q)$. Apply this with
$f(\tau) := \mathbb{E}_{\bm{y}\sim\pi(\cdot\mid\tau(\bm{x}))}
[R(\bm{y},\tau(\bm{x}))]$
(whose oscillation is at most $\Delta_R$), $P = P_{\text{test}}$,
$Q = \rho_{\bm{x}}$: the bound is $\Delta_R\,\eta$.
\end{proof}

\paragraph{Tightness of $\eta\Delta_R$.}
For the TV transfer inequality itself, the bound is sharp even with
non-degenerate coverage. Let
$\rho_{\bm{x}} = (1-\gamma)\delta_{\tau_1}+\gamma\delta_{\tau_2}$ for
$\gamma\in(0,1)$, and let
$P_{\text{test}} =
(1-\gamma-\eta)\delta_{\tau_1}+\gamma\delta_{\tau_2}+\eta\delta_{\tau_3}$
with $\tau_3\notin\{\tau_1,\tau_2\}$ and $\eta\le 1-\gamma$. Construct $\pi$
so that its answer marginals on $\tau_1$ and $\tau_2$ are identical (hence it
is $0$-C2-feasible on the support of $\rho_{\bm{x}}$), with
$\mathcal{R}(\pi,\tau_1;\bm{x})
=\mathcal{R}(\pi,\tau_2;\bm{x})=R_{\max}$, and let
$\mathcal{R}(\pi,\tau_3;\bm{x})=R_{\min}$ outside the coverage support. Then
$\mathrm{TV}(P_{\text{test}},\rho_{\bm{x}})=\eta$ and the left-hand side of
Eq.~\eqref{eq:coverage_transfer} equals $\eta\Delta_R$.

\section{Embedding-Space Adversary Details}
\label{sec:appendix_adv}

The main adversary sampler $q_\phi$ perturbs the instruction in embedding space: it
adds a bounded perturbation $\bm{\delta}$ to the instruction-token embeddings,
$\tilde{\bm{e}}_{\mathcal{I}}(\bm{x})
= \bm{e}_{\mathcal{I}}(\bm{x}) + \bm{\delta}$ with
$\|\bm{\delta}\|_2 \le \epsilon_{\mathrm{emb}}$, and minimizes the policy's
expected DTR reward within this instruction-embedding ball. In the
projected-gradient
implementation below, $q_\phi$ is the point mass at the final perturbation;
with a stochastic perturbation head, expectations over $q_\phi$ are over
reparameterized perturbation samples. We write $\tau_\phi(\bm{x})$ for the
prompt carrying the perturbed instruction embedding. The parser/reward keeps
the original surface delimiter for embedding perturbations,
$d_{\tau_\phi}=d_{\bm{x}}$.

\paragraph{Estimator.} For $G>1$, treating the reward evaluation as a stop-gradient with
respect to the perturbation, the score-function gradient of the negative reward
with respect to the directly optimized instruction perturbation $\bm{\delta}$
(or a reparameterized perturbation sample, after the chain rule to $\phi$) is
estimated by
\begin{equation}
  \hat{\bm{g}} \approx \frac{1}{G}\sum_{i=1}^{G}
    (\bar R_{-i}-R_i)\,
    \nabla_{\bm{\delta}}\log\pi_\theta(\bm{y}_i\mid\tau_\phi(\bm{x})),
  \label{eq:ggds_reinforce}
\end{equation}
where $R_i = R_{\mathrm{DTR}}(\bm{y}_i,\tau_\phi(\bm{x}))$ and
$\bar R_{-i} = \frac{1}{G-1}\sum_{j\ne i}R_j$ is a leave-one-out baseline
independent of $\bm{y}_i$ conditional on the current perturbation. The sign uses
$\bar R_{-i}-R_i$ because the
adversary ascends the negative reward. If the reward function is differentiated
directly through $\tau_\phi$, an additional pathwise term
$-\nabla_{\bm{\delta}}R_{\mathrm{DTR}}$ appears; our implementation treats the
verifier output as non-differentiable and uses the score-function term as the
attack direction.

\paragraph{Objective.} The adversary maximizes the negation of the policy's
expected DTR reward over the $\epsilon_{\mathrm{emb}}$-ball:
\begin{equation}
\begin{aligned}
  \max_\phi\;
  &\mathbb{E}_{\bm{x}\sim\mathcal{D}}\!
    \Bigl[
      -\,\mathbb{E}_{\tau\sim q_\phi(\cdot\mid\bm{x})}\,
       \mathbb{E}_{\bm{y}\sim\pi_\theta(\cdot\mid\tau(\bm{x}))} \
       R_{\mathrm{DTR}}(\bm{y},\tau(\bm{x}))
    \Bigr],
  \label{eq:adv_obj}
\end{aligned}
\end{equation}
where $\tau\sim q_\phi(\cdot\mid\bm{x})$ denotes sampling a bounded
perturbation $\bm{\delta}$ and applying it to the instruction embedding.

\paragraph{Perturbation set.}
The main perturbation acts on the instruction-token embedding block. We use an
$\ell_2$ trust region of radius $\epsilon_{\mathrm{emb}}$ and
project $\bm{\delta}$ back onto
$\{\|\bm{\delta}\|_2\le\epsilon_{\mathrm{emb}}\}$ after each step. An
optional entropy regularizer on a stochastic $q_\phi$ (a perturbation head
with learned mean and diagonal covariance) keeps the adversary from collapsing
onto a single direction, so the policy can be trained against a spread of
worst-case perturbations rather than one fixed attack. The image content,
question content, and answer target are never modified; only the instruction
embedding block is perturbed in the main method.

\paragraph{Algorithm.}
At each adversarial step on input $\bm{x}$ with current policy $\pi_\theta$:

\begin{algorithm}[t]
  \caption{Embedding-space adversary one-step (per prompt $\bm{x}$)}
  \footnotesize
  \begin{algorithmic}[1]
    \STATE \textbf{Input:} prompt $\bm{x}$, radius $\epsilon_{\mathrm{emb}}$, ascent steps $K_{\mathrm{adv}}$, step size $\eta_{\mathrm{adv}}$, adversary group size $G_{\mathrm{adv}}$, current policy $\pi_\theta$.
    \STATE Initialize $\bm{\delta} \leftarrow \bm{0}$.
    \FOR{$j = 1, \dots, K_{\mathrm{adv}}$}
      \STATE Sample a rollout group $\{\bm{y}_i\}_{i=1}^{G_{\mathrm{adv}}} \sim \pi_\theta(\cdot \mid \tau_\phi(\bm{x}))$ at the current perturbation.
      \STATE Estimate $\hat{\bm{g}}$ via Eq.~\eqref{eq:ggds_reinforce} (REINFORCE on $-R_{\mathrm{DTR}}$).
      \STATE Ascend and project: $\bm{\delta} \leftarrow \Pi_{\|\cdot\|_2\le\epsilon_{\mathrm{emb}}}\!\bigl(\bm{\delta} + \eta_{\mathrm{adv}}\,\hat{\bm{g}}\bigr)$.
    \ENDFOR
    \STATE \textbf{Return} perturbed prompt $\tau_\phi(\bm{x})$ with instruction embedding $\bm{e}_{\mathcal{I}}(\bm{x})+\bm{\delta}$.
  \end{algorithmic}
\end{algorithm}

\paragraph{Hyperparameters.}
We set the trust-region radius $\epsilon_{\mathrm{emb}}$ to $10\%$ of the median
instruction-embedding norm, ascent steps $K_{\mathrm{adv}} = 4$, and a rollout
size $G = 4$ for the gradient estimate (a smaller group than the training group
$G = 16$, to keep adversary cost bounded). The adversary update is performed
every $4$ policy updates to amortize cost.

\paragraph{Why embedding-space.}
Perturbing the instruction embedding is a differentiable, model-agnostic surrogate
for the worst-case rewording in $\llbracket\bm{x}\rrbracket$: it needs no
candidate enumeration and directly models continuous instruction-embedding variation,
at the cost that an arbitrary $\bm{\delta}$ need not decode to a literal token
sequence. A discrete,
on-manifold variant that
searches a semantics-preserving substitution pool with gradient guidance is naturally an extension that we leave to the future work.

\paragraph{Compute cost.}
Per prompt per adversary step, the adversary requires: (a) one rollout of size
$G$ to estimate the REINFORCE gradient $\hat{\bm{g}}$
(Eq.~\eqref{eq:ggds_reinforce}); if this group is sampled at the final
perturbation, it can be reused by the policy update, otherwise it is counted as
additional adversary sampling cost; (b) one backward pass to the
instruction-embedding block, $O(|\text{instruction}|\cdot d)$; repeated over
$K_{\mathrm{adv}}$ ascent steps. With $K_{\mathrm{adv}} = 4$ this is roughly
$K_{\mathrm{adv}}$ forward--backward-pass-equivalents per adversary update.
Running the adversary update once every 4 policy updates limits the measured
wall-clock overhead to about 25\% above plain GRPO on Qwen2.5-VL-7B.

\subsection{Image- vs instruction-side perturbation}
\label{sec:appendix_img_text}

Our main adversary perturbs the instruction (text) embedding only
(\mainsec{3.2}). Fig.~\ref{fig:perturb_ablation} ablates the
training-time perturbation modality (image, text, or both) against test-time
perturbations on Qwen2.5-VL-7B. The instruction-side (Text) adversary recovers
most of the robustness to prompt rewording---the Text and joint (Both)
columns---whereas robustness to image-side perturbation requires image-side
adversarial training, and jointly perturbing both inputs is only marginally
better than the text adversary on the rewording columns. We therefore adopt the
instruction-side adversary in the main method and report the image-side variant
here for completeness.

\begin{figure*}[t]
  \centering
  \includegraphics[width=0.8\textwidth]{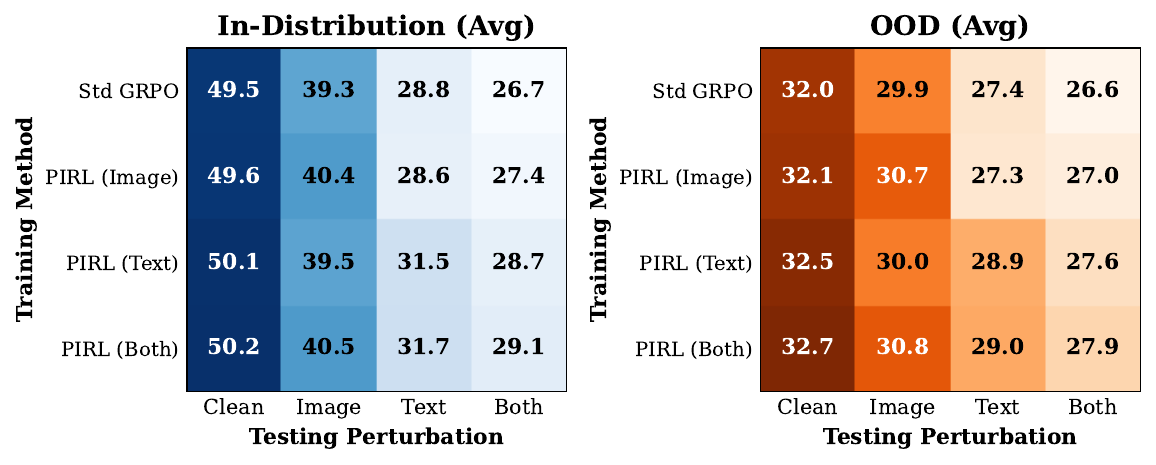}
  \caption{Ablation on training--testing perturbation combinations for
    Qwen2.5-VL-7B (in-distribution and OOD averages). Rows vary the adversarial
    perturbation applied during training (none for Std GRPO; image, text, or
    both for PIRL); columns vary the test-time perturbation, where ``Clean''
    applies none and ``Both'' applies joint image+text perturbation. The
    instruction-side (Text) adversary of our main method drives robustness to
    prompt rewording (Text and Both columns), whereas image robustness requires
    image-side adversarial training.}
  \label{fig:perturb_ablation}
\end{figure*}

\subsection{Full PIRL training loop}
\label{sec:appendix_alg}

Algorithm~\ref{alg:pirl} summarizes the complete PIRL training procedure. For
each task instance, full PIRL first samples
$\tau_0\sim P_{\mathrm{train}}$ from the five-template training pool, then
constructs an adversarial instruction perturbation around $\tau_0(\bm{x})$
through projected ascent on the negative DTR reward, and finally samples a
training rollout group under the resulting perturbed prompt. The no-MT ablation
uses $\tau_0=\mathrm{id}$. The corresponding DTR scores and DAN-normalized
advantages are used in the clipped policy objective of \maineq{10}, together
with the token-level consistency loss in \maineq{8} that regularizes the policy
distributions under the clean template-conditioned and adversarial prompts.
During the policy update, the generated adversarial prompt is treated as fixed.

As formalized in Definition~\ref{def:ie}, C2 promotes answer-level invariance across semantics-preserving prompt transformations by constraining the expected (f)-divergence between the answer marginals induced by transformation-dependent extractors under a non-degenerate coverage distribution. This formulation permits variations in prompt wording and required output format while encouraging the policy to preserve the same underlying semantic answer. As shown in Proposition~\ref{prop:c1c2_consequences}, sufficiently strong invariance control excludes format-dependent shortcut solutions within the covered transformation set. In practice, we optimize the length-normalized token-level consistency loss in \maineq{8} as a tractable differentiable surrogate for C2, rather than treating it as a formal certificate of answer-level invariance. The complete PIRL training procedure is presented in Algorithm~\ref{alg:pirl}.

\subsection{Instrumentation for the empirical \texorpdfstring{$\Phi_{\mathrm{wf}}$}{Phi-wf} trajectory}
\label{sec:appendix_ggds_metrics}

The hooks below produce a single-run diagnostic check of
Theorems~\ref{thm:snr} and~\ref{thm:dtr_decouples} without re-training. For
each rollout $\bm{y}_i$ in a
prompt group, log $V_f(\bm{y}_i, \mathcal{I})$ from the format checker and
the realized reward $R(\bm{y}_i, \bm{x})$ (binary $V$ for GRPO;
$R_t + \alpha_{\mathrm{rew}} R_s$ for
PIRL). At each training step, aggregate per-prompt-group estimates of the
conditional reward variances $\widehat{\mathrm{Var}}(R \mid V_f = v)$ for
$v \in \{0,1\}$ using Bessel-corrected sample variance when the bucket contains
at least two samples; buckets with fewer than two samples are pooled across
prompt groups at the same step before variance estimation, and the step's
$\widehat{\Phi}_{\mathrm{wf}}$ diagnostic is omitted if no pooled estimate is available for
each required realized bucket. Then form
$\widehat{\Phi}_{\mathrm{wf}} := \widehat{\mathbb{E}}[\widehat{\mathrm{Var}}(R\mid V_f)] / \widehat{\mathrm{Var}}(R)$.
Under the independence and approximately fixed-$p_c$ diagnostic model of
Theorem~\ref{thm:snr}, the GRPO trajectory should rise toward $1$ as
$\hat p_f \to 1$. Under the non-degeneracy assumption of
Theorem~\ref{thm:dtr_decouples}, PIRL has a $p_f$-independent numerator floor
$\alpha_{\mathrm{rew}}^{2}m$ for the reward-variance share; combining it with
Popoviciu's inequality gives the distribution-uniform constant in
Theorem~\ref{thm:dtr_decouples}, while
$\alpha_{\mathrm{rew}}^{2}m/\widehat{\mathrm{Var}}(R)$
is only a data-dependent diagnostic floor. These are diagnostic checks of the
theory's assumptions, not standalone proofs of training dynamics.

\begin{algorithm}[t]
  \caption{PIRL training pipeline}
  \label{alg:pirl}
  \footnotesize
  \begin{algorithmic}[1]
    \STATE \textbf{Input:} policy $\pi_\theta$, dataset $\mathcal{D}$, template distribution $P_{\mathrm{train}}$, radius $\epsilon_{\mathrm{emb}}$, ascent steps $K_{\mathrm{adv}}$, step size $\eta_{\mathrm{adv}}$, group sizes $G_{\mathrm{adv}},G_{\mathrm{train}}$, refresh cadence $C_{\mathrm{adv}}$, weight $\lambda$.
    \FOR{each adversary refresh (once per $C_{\mathrm{adv}}$ policy updates)}
      \STATE Set rollout policy $\theta_{\mathrm{old}}\leftarrow\theta$.
      \STATE Sample $\mathcal{B} = \{(\bm{x}, a^*)\} \sim \mathcal{D}$ and $\tau_0\sim P_{\mathrm{train}}$ per instance; form $\mathcal{B}_0=\{(\bm{x}_0,a^*):\bm{x}_0=\tau_0(\bm{x})\}$.
      \FOR{each $\bm{x}_0 \in \mathcal{B}_0$}
        \STATE Initialize $\bm{\delta} \leftarrow \bm{0}$.
        \STATE Let $\tau_\phi(\bm{x}_0)$ denote the prompt induced by the current $\bm{\delta}$.
        \FOR{$j = 1, \dots, K_{\mathrm{adv}}$}
          \STATE Sample $\{\bm{y}_i\}_{i=1}^{G_{\mathrm{adv}}} \sim \pi_{\theta_{\mathrm{old}}}(\cdot\mid \tau_\phi(\bm{x}_0))$; estimate $\hat{\bm{g}}$ via Eq.~\eqref{eq:ggds_reinforce}; ascend and project $\bm{\delta} \leftarrow \Pi_{\|\cdot\|_2\le\epsilon_{\mathrm{emb}}}(\bm{\delta} + \eta_{\mathrm{adv}}\,\hat{\bm{g}})$.
          \STATE Update $\tau_\phi(\bm{x}_0)$ to use instruction embedding $\bm{e}_{\mathcal{I}}(\bm{x}_0)+\bm{\delta}$.
        \ENDFOR
        \STATE $\tau_\phi(\bm{x}_0) \leftarrow$ prompt with instruction embedding $\bm{e}_{\mathcal{I}}(\bm{x}_0)+\bm{\delta}$; sample $\{\tilde{\bm{y}}_i\}_{i=1}^{G_{\mathrm{train}}} \sim \pi_{\theta_{\mathrm{old}}}(\cdot\mid \tau_\phi(\bm{x}_0))$.
        \STATE Compute trinary/semantic scores from Eqs.~(5) and~(6) in the main paper and DAN advantages $\hat A_i$ from \maineq{7}.
        \STATE Store the frozen adversarial prompt $\tau_{\bar\phi_{\bm{x}_0}}(\bm{x}_0) \leftarrow \tau_\phi(\bm{x}_0)$ and its rollout group.
      \ENDFOR
      \FOR{$c=1,\ldots,C_{\mathrm{adv}}$}
        \STATE Compute the batch consistency loss $|\mathcal{B}_0|^{-1}\sum_{\bm{x}_0\in\mathcal{B}_0}\mathcal{L}_{\mathrm{cons}}(\theta;\bm{x}_0,\bar\phi_{\bm{x}_0})$ via \maineq{8}, using teacher-forced prefixes in our implementation.
        \STATE Update $\pi_\theta$ via \maineq{10}, using the stored adversarial prompts and treating them as fixed.
      \ENDFOR
    \ENDFOR
  \end{algorithmic}
\end{algorithm}

\section{Dynamic Evaluation: Operator Set and Class-Decomposed Results}
\label{sec:appendix_dynamic_eval}

We implement dynamic evaluation by applying controlled transformations intended
to preserve the task label, building on Dynamic Multimodal Evaluation with
Vision-Language Bootstrapping~\citep{yang2025dynamicmmeval}. Unlike T-Stress,
D-Test is not restricted to prompt transformations in $\mathcal{T}$: image,
layout, and QA-format operations can change task difficulty and may occasionally
alter answer-relevant content. We therefore treat D-Test as a broader empirical
robustness stress test, not as a clean evaluation of the C2 equivalence class.

\paragraph{Operator pool.}
The D-Test applies the following mutations to input instances:
\begin{itemize}
  \item \textbf{Text:}
    rephrase (paraphrase or reorder wording without changing meaning);
    narrative distraction (add plausible but irrelevant sentences);
    option shuffle / equivalent replacement; format switch
    (QA $\leftrightarrow$ MC); choice expansion (e.g., ABCD $\to$ ABCDEF);
    answer negation (``None of the above'').
  \item \textbf{Image:}
    resize+padding (rescale with aspect-preserving padding);
    crop (center/random crop with ROI safeguards);
    rotation when the question is not orientation-sensitive;
    salt-and-pepper noise; color inversion when color is not the queried
    attribute; style transfer (lightweight filter-based style shift). These
    safeguards reduce, but do not eliminate, label-change risk.
\end{itemize}

\paragraph{Implementation.}
Our integration adds a dynamic/original branch to each per-benchmark loader in
the evaluation harness; the aggregate D-Test in
\maintable{2} averages over the full
operator pool.

\section{Supplementary Method Details}
\label{sec:appendix_supp_method}

\subsection{RLVR and GRPO Preliminaries}

We define the MLLM as a policy $\pi_\theta$ that generates a response
$\bm{y} = (y_1, \dots, y_T)$ conditioned on an input prompt $\bm{x}$. A verifier $V$
deterministically assigns a binary reward $r = V(\bm{y}, \bm{x}) \in \{0,1\}$, where
$r = 1$ indicates that the final answer satisfies task-specific correctness
constraints (e.g., template-based numeric matching). The RLVR objective
(e.g., GRPO) maximizes the expected verifiable reward:
\begin{equation}
  \max_{\theta} \;
  \mathbb{E}_{\bm{x} \sim \mathcal{D}}\,
  \mathbb{E}_{\bm{y} \sim \pi_\theta(\cdot \mid \bm{x})}
  \big[V(\bm{y},\bm{x})\big],
\end{equation}
where $\mathcal{D}$ denotes the empirical prompt distribution.

\paragraph{Group Relative Policy Optimization (GRPO).}
GRPO is a RLVR algorithm that optimizes policies using group-wise
relative rewards. Given a prompt $\bm{x}$, GRPO samples a group
$\bm{Y}=\{\bm{y}^{(i)}\}_{i=1}^G$ from the behavior policy
$\pi_{\theta_{\mathrm{old}}}$ and computes one normalized advantage per
response from the group's verifier outcomes. For
$\bm{y}^{(i)}=(y^{(i)}_1,\ldots,y^{(i)}_{T_i})$, the clipped surrogate is
token-level, as in PPO~\citep{schulman2017proximal}:
\begin{equation}
\begin{aligned}
  q_{\mathrm{old}}^G(\cdot\mid\bm{x})
  &:=
  \pi_{\theta_{\mathrm{old}}}^{\otimes G}(\cdot\mid\bm{x}),\\
  w_{i,t}(\theta)
  &:=
  \frac{\pi_\theta(y_t^{(i)}\mid\bm{x},\bm{y}_{<t}^{(i)})}
       {\pi_{\theta_{\mathrm{old}}}(y_t^{(i)}
          \mid\bm{x},\bm{y}_{<t}^{(i)})},\\
  c_{i,t}(\theta)
  &:=
  \mathrm{clip}\!\left(
    w_{i,t}(\theta), 1-\epsilon_{\mathrm{low}},
    1+\epsilon_{\mathrm{high}}
  \right),\\
  M_{i,t}(\theta)
  &:=
  \min\!\left(w_{i,t}(\theta) A^{(i)},
              c_{i,t}(\theta) A^{(i)}\right),\\
  d^{\mathrm{ref}}_{i,t}(\theta)
  &:=
  D_{\mathrm{KL}}\!\left(
    \pi_\theta(\cdot\mid\bm{x},\bm{y}_{<t}^{(i)})
    \,\middle\|\,
    \pi_{\mathrm{ref}}(\cdot\mid\bm{x},\bm{y}_{<t}^{(i)})
  \right),\\
  B_\theta(\bm{x},\bm{Y})
  &:=
  \frac{1}{G}\sum_{i=1}^{G}\frac{1}{T_i}
  \sum_{t=1}^{T_i}
  \left[M_{i,t}(\theta)-\beta d^{\mathrm{ref}}_{i,t}(\theta)\right].
\end{aligned}
\end{equation}
\begin{equation}
\begin{aligned}
  \mathcal{J}_{\text{GRPO}}
  &=
  \mathbb{E}_{\bm{x}\sim\mathcal{D}}\,
  \mathbb{E}_{\bm{Y}\sim q_{\mathrm{old}}^G(\cdot\mid\bm{x})}
  \left[B_\theta(\bm{x},\bm{Y})\right].
\end{aligned}
\end{equation}
The reference policy appears separately in the optional KL penalty. Here
$\epsilon_{\mathrm{low}}$ and $\epsilon_{\mathrm{high}}$ are the lower and
upper clipping radii and $\beta$ is the KL strength; the expectation is over
$\bm{x}\sim\mathcal{D}$ and old-policy rollout groups. $A^{(i)}$ is the
group-relative advantage:
\begin{equation}
  A^{(i)}
    = \frac{r^{(i)} - \mu_r}{\sigma_r + \epsilon_{\mathrm{norm}}},
    \label{eq:grpo_advantage}
\end{equation}
\begin{equation}
  \mu_r
    = \frac{1}{G}\sum_{j=1}^{G} r^{(j)},
\end{equation}
\begin{equation}
  \sigma_r^2
    = \frac{1}{G}\sum_{j=1}^{G} (r^{(j)} - \mu_r)^2,
\end{equation}
with $\epsilon_{\mathrm{norm}} > 0$ a small constant for numerical stability.

\paragraph{Generalized loss for the prompt-invariance objective.}
As defined in \maineq{10}, we write
$\mathcal{L}_{\mathrm{G}}^{\mathrm{DAN}}(\pi_\theta, q; R_t,R_s)$ for the
clipped GRPO loss in which transformations are sampled from the adversarial
sampler $q$ for a fixed task instance $\bm{x}$; this $q$ is not the C2 coverage
distribution $\rho_{\bm{x}}$. Advantages are computed from the decomposed
rewards by \maineq{7}. Define
\begin{equation}
\begin{aligned}
w_{i,t}(\theta, \tau)
&:= \frac{\pi_\theta(y_t^{(i)}
          \mid\tau(\bm{x}),\bm{y}_{<t}^{(i)})}
        {\pi_{\theta_{\mathrm{old}}}(y_t^{(i)}
          \mid\tau(\bm{x}),\bm{y}_{<t}^{(i)})},\\
c_{i,t}(\theta,\tau)
&:= \mathrm{clip}\!\left(
      w_{i,t}(\theta,\tau), 1-\epsilon_{\mathrm{low}},
      1+\epsilon_{\mathrm{high}}
    \right),\\
d^{\mathrm{ref}}_{i,t}(\theta,\tau)
&:= D_{\mathrm{KL}}\!\left(
      \pi_\theta(\cdot\mid\tau(\bm{x}),\bm{y}_{<t}^{(i)})
      \,\middle\|\,
      \pi_{\mathrm{ref}}(\cdot\mid\tau(\bm{x}),\bm{y}_{<t}^{(i)})
    \right),\\
M_{i,t}(\theta,\tau)
&:= \min\!\left(w_{i,t}(\theta,\tau)\hat A_i,\,
               c_{i,t}(\theta,\tau)\hat A_i\right),\\
\bar M_{\theta}(\tau,\bm{Y})
&:= \frac{1}{G}\sum_{i=1}^{G}\frac{1}{T_i}
    \sum_{t=1}^{T_i}M_{i,t}(\theta,\tau),\\
\bar D_{\mathrm{ref}}(\theta,\tau,\bm{Y})
&:= \frac{1}{G}\sum_{i=1}^{G}\frac{1}{T_i}
    \sum_{t=1}^{T_i}d^{\mathrm{ref}}_{i,t}(\theta,\tau),\\
q_\tau^G
&:= \pi_{\theta_{\mathrm{old}}}^{\otimes G}
    (\cdot\mid\tau(\bm{x})).
\end{aligned}
\end{equation}
Write
$\mathcal{L}_{\mathrm{DAN}}
:=\mathcal{L}_{\mathrm{G}}^{\mathrm{DAN}}(\pi_\theta,q;R_t,R_s)$.
\begin{equation}
\begin{aligned}
\mathcal{L}_{\mathrm{DAN}}
  &= -\,\mathbb{E}_{\tau\sim q}\,
    \mathbb{E}_{\bm{Y}\sim q_\tau^G}\!
    \left[\bar M_{\theta}(\tau,\bm{Y})\right] + \beta\,\mathbb{E}_{\tau\sim q}\,
    \mathbb{E}_{\bm{Y}\sim q_\tau^G}\!
    \left[\bar D_{\mathrm{ref}}(\theta,\tau,\bm{Y})\right].
\end{aligned}
\end{equation}
$\hat A_i$ is the DAN-normalized advantage of \maineq{7} on
the pair $(R_t,R_s)$, conditional on the sampled minibatch used by
$\mathrm{Norm}_{\mathcal{B}}$. Thus the practical policy-gradient objective is not
ordinary GRPO on the scalar reward alone; DAN is how DTR is plugged into GRPO.
We do not isolate DAN in the ablation.

\subsection{Interpretation of DTR}

DTR combines a format-aware term and a semantic-score term additively,
$R = R_t + \alpha_{\mathrm{rew}} R_s$.
The trinary $R_t$ distinguishes correct-and-formatted, wrong-but-formatted,
and format-failure cases, rather than collapsing all non-$(1,1)$ outcomes to
the same $V=0$ value. The score $R_s$ varies within each $V_f$ bucket under
the non-degeneracy assumption; when it is aligned with correctness, it can rank
responses inside those buckets. A
response with correct semantic reasoning but format failure still attains a
total reward of $-1 + \alpha_{\mathrm{rew}} R_s$, which can remain negative but is strictly
greater than $-1$ whenever $R_s > 0$; when $R_s$ varies across such responses,
the reward signal contains score variation within the format-failure bucket,
which the binary verifier collapses to a single $V=0$ value.
Theorem~\ref{thm:dtr_decouples}
gives the corresponding reward-variance lower bound.

\section{Datasets, Licenses, and Reproducibility}
\label{sec:appendix_repro}

This section consolidates dataset statistics, artifact licenses and intended
use, the compute budget, full hyperparameters, and software versions, for
reproducibility and responsible-research reporting.

\subsection{Datasets and Statistics}
\label{sec:appendix_data_stats}
Table~\ref{tab:datasets} summarizes the benchmarks, their role
(training, in-distribution, or out-of-distribution evaluation), domain, and
evaluation-set size. Exam-VQA RL training draws prompts from the MMK12 pool
($15{,}616$ image--question pairs); each online-RL iteration samples $200$
prompts with $G=16$ rollouts per prompt over $10$ iterations. Medical and legal
policies are trained on Lingshu~\citep{xu2025lingshu} and
DUDE~\citep{vanlandeghem2023dude}, respectively. Exam evaluation-set sizes are
measured directly from our evaluation harness; medical and legal sizes are the
standard public test/validation splits.

\begin{table*}[t]
  \centering
  \footnotesize
  \caption{Datasets used, with role, domain, number of instances used, and
  license. MMK12 reports train/evaluation counts; for the other Train rows the
  count is the training subset used, and $\dagger$ marks a standard public
  evaluation split. Exam evaluation sizes are measured from our harness.
  Licenses are reported to the best of our knowledge as distributed by the
  original sources.}
  \label{tab:datasets}
  \setlength{\tabcolsep}{6pt}
  \begin{tabular}{@{}lllrl@{}}
    \toprule
    \textbf{Dataset} & \textbf{Role} & \textbf{Domain} & \textbf{\#Used} & \textbf{License (as distributed)} \\
    \midrule
    MMK12~\citep{meng2025mm}                 & Train / ID eval & Exam (math/science)   & $15{,}616/2{,}000$        & Research use \\
    MathVista~\citep{lu2024mathvista}        & ID eval         & Exam (math)           & $1{,}000$        & CC BY-SA 4.0 \\
    Olym-Phys~\citep{he2024olympiadbench}    & OOD eval        & Exam (physics)        & $451$            & MIT \\
    \midrule
    Lingshu~\citep{xu2025lingshu}            & Train           & Medical               & $44{,}050$              & Research use \\
    VQA-RAD~\citep{lau2018dataset}           & ID eval         & Medical (radiology)   & $451^{\dagger}$  & CC0 1.0 \\
    PathVQA~\citep{he2020pathvqa}            & ID eval         & Medical (pathology)   & $6{,}761^{\dagger}$ & MIT \\
    GMAI-MMBench~\citep{chen2024gmaimmbench} & OOD eval        & Medical               & $5{,}000^{\dagger}$ & CC BY-NC-SA 4.0 \\
    \midrule
    DUDE~\citep{vanlandeghem2023dude}        & Train           & Legal / document      & $5{,}000$             & CC BY 4.0 \\
    DocVQA~\citep{mathew2021docvqa}          & OOD eval        & Legal / document      & $5{,}349^{\dagger}$ & Research use \\
    \bottomrule
  \end{tabular}
\end{table*}

\subsection{Licenses and Intended Use}
\label{sec:appendix_licenses}
Both base models, Qwen2.5-VL-7B-Instruct and Qwen3-VL-8B-Instruct, are released
under Apache-2.0. Dataset licenses are listed in Table~\ref{tab:datasets}. All
artifacts are used for non-commercial academic research consistent with their
intended research use; benchmarks accessed under research-only terms (medical and
legal VQA) are used solely in a research context and are not redistributed.
The artifacts we create---trained policy checkpoints and dynamic-evaluation
operators---are intended for research use only. Any release will carry its own
license and notices; users remain responsible for the terms of the source
datasets, which we do not redistribute.

\subsection{Models, Compute, and Infrastructure}
\label{sec:appendix_compute}
We train two instruction-tuned multimodal policies, Qwen2.5-VL-7B-Instruct and
Qwen3-VL-8B-Instruct ($7$--$8$B parameters each, including the
vision encoder). Training runs on $16\times$ NVIDIA H100-80GB GPUs with DeepSpeed
ZeRO-3 and bf16 mixed precision; policy rollouts are served with vLLM. The
embedding-space adversary adds about $25\%$ wall-clock over plain GRPO
(Section~\ref{sec:appendix_adv}); the total budget per trained policy is on the
order of a few hundred H100 GPU-hours for the $10$-iteration online-RL loop.

\subsection{Training Hyperparameters}
\label{sec:appendix_hparams}
Table~\ref{tab:hparams} lists the policy-optimization, sampling, and reward
settings. The trinary reward $R_t$ is combined with the semantic score $R_s$ at
weight $\alpha_{\mathrm{rew}}=0.01$; the consistency penalty (the
adversarial-to-clean token KL in \mainsec{3.2}) has
weight $\lambda=1.0$, and is distinct from the reference-KL coefficient, which we
set to $\beta=0$; we instead rely on the asymmetric $0.2/0.28$ clipping rule
to regularize the policy update, without claiming a hard trust-region bound.
The adversary trust-region radius and ascent schedule are given in
Section~\ref{sec:appendix_adv}. We did not run an extensive hyperparameter search:
optimizer and clip settings follow GRPO's settings, and $\alpha_{\mathrm{rew}}$,
$\lambda$, and the adversary radius were set by small manual sweeps on the Exam
validation split.

\begin{table*}[t]
  \centering
  \footnotesize
  \caption{Training, rollout, and reward hyperparameters (GRPO with DTR/DAN and
  the embedding-space adversary). Adversary radius and ascent cadence are in
  Section~\ref{sec:appendix_adv}.}
  \label{tab:hparams}
  \setlength{\tabcolsep}{6pt}
  \begin{tabular}{@{}llll@{}}
    \toprule
    \textbf{Hyperparameter} & \textbf{Value} & \textbf{Hyperparameter} & \textbf{Value} \\
    \midrule
    Policy optimizer        & GRPO              & Rollout group size $G$         & $16$ \\
    Learning rate           & $2\times10^{-6}$  & MM-projector LR                & $1\times10^{-5}$ \\
    Optimizer               & AdamW             & $(\beta_1,\beta_2,\epsilon)$   & $(0.9,0.999,10^{-8})$ \\
    Weight decay            & $0.01$            & Max grad norm                  & $1.0$ \\
    Clip range (low/high)   & $0.2 / 0.28$      & Reference-KL coef.\ $\beta$     & $0$ \\
    Online-RL iterations    & $10$              & Prompts per iteration          & $200$ \\
    Policy steps per iter.  & $\le 50$          & Per-device batch               & $2$ \\
    Grad.\ accumulation     & $2$               & Max sequence length            & $8192$ \\
    Max new tokens (rollout)& $4096$            & Precision                      & bf16 \\
    Rollout temperature     & $0.9$             & Rollout top-$p$                & $0.99$ \\
    $\alpha_{\mathrm{rew}}$ (semantic) & $0.01$  & $\lambda$ (consistency)        & $1.0$ \\
    Adversary LR $\eta_{\mathrm{adv}}$ & $5\times10^{-5}$ & ZeRO stage          & $3$ \\
    Tuned modules           & LLM\,+\,MLP       & Frozen modules                 & vision enc. \\
    Random seed             & $42/43/44$              & Eval decoding                  & greedy ($T{=}0$) \\
    \bottomrule
  \end{tabular}
\end{table*}

\subsection{Software and Packages}
\label{sec:appendix_software}
Training uses a custom Qwen-VL RL framework (PyTorch~2.7, Transformers~4.57.x,
DeepSpeed~0.16, FlashAttention~2.7.4) with vLLM for distributed rollout
generation. Evaluation uses our improved MedEvalKit harness with vLLM and
Transformers~4.57.x. Its scoring stack includes \texttt{nltk}~3.9.1,
\texttt{rouge}~1.0.1, \texttt{mathruler}~0.1.0, and
\texttt{qwen-vl-\allowbreak utils}: exact-match and numeric answers are scored with
\texttt{mathruler}, and open-ended answers are graded by an LLM judge
(GPT-4.1-nano) under greedy decoding. Unless otherwise noted, evaluation
generation uses temperature $0$, top-$p$ $0.01$, at most $1024$ new tokens, and
pass@$1$.

\section{Potential Risks}
\label{sec:appendix_risks}
Our method centers on an adversary that perturbs instruction embeddings to
surface worst-case rewordings, together with a suite of controlled input
operators for dynamic evaluation. This machinery is dual-use: rather than
hardening a model, an adversary could in principle repurpose the embedding-space
search or the operator suite to \emph{craft} prompt perturbations that degrade a
deployed model, or to probe for inputs on which it is brittle.

We mitigate this risk in several ways. First, the technique is fundamentally
defensive: its objective is to \emph{reduce} the prompt-induced performance gap,
and the released recipe strengthens robustness rather than providing a turnkey
attack. Second, the main adversary operates in continuous embedding space, and
its perturbations need not decode to literal, transferable text prompts
(\mainsec{3.2}; Limitations), so it is not directly usable as a
black-box textual attack. Third, the dynamic-evaluation operators are controlled
stress transformations meant for measuring robustness, not for generating
misleading or harmful content; as noted in
Appendix~\ref{sec:appendix_dynamic_eval}, they are not all certified members of
the semantic-equivalence class. We therefore release code and
operators for research use only, documented as evaluation and
robustness-training tools, and we recommend pairing any deployment with standard
input validation and robustness monitoring.

\begin{table}[h]
  \centering
  \caption{Pass@8 accuracy under standard evaluation (\%). Cited in \mainsec{4.6}.}
  \label{tab:pass8_results}
  \resizebox{0.65\linewidth}{!}{
    \begin{tabular}{lc|ccc}
      \toprule
      \textbf{Benchmark} & \textbf{OOD} & \textbf{Qwen2.5-VL-7B} & \textbf{GRPO} & \textbf{PIRL (Ours)} \\
      \midrule
      VQA\_RAD & $\times$ & 76.99 & 75.66 & \textbf{77.43} \\
      PathVQA & $\times$ & 58.81 & 58.22 & \textbf{58.86} \\
      GMAI-MMbench & $\checkmark$ & 44.85 & 42.39 & \textbf{44.90} \\
      \bottomrule
    \end{tabular}
  }
\end{table}

Finally, robustness is not correctness: a prompt-stable model can still be
confidently wrong. In the high-stakes medical and legal domains studied here,
PIRL-trained policies are intended to support, not replace, expert human
judgment, and should not be deployed as autonomous decision-makers.

\section{Pass@8 Results}
\label{sec:appendix_pass8}

\begin{figure}[H]
  \centering
  \includegraphics[width=0.6\linewidth]{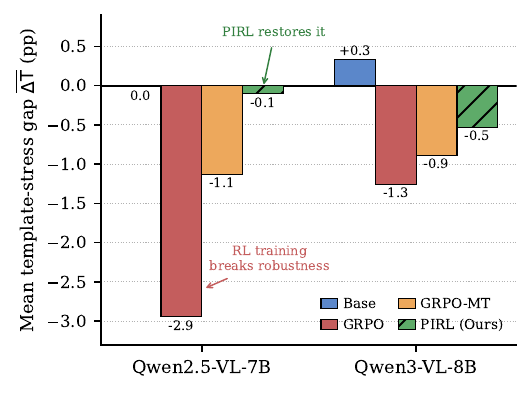}
  \caption{\textbf{Mean template-stress gap} $\overline{\Delta\mathrm{T}}$
  (over 7 benchmarks; closer to $0$ is more robust), recomputed from
  Table 1. Single-template GRPO
  sharply degrades robustness (Qwen2.5: $\approx\!0\!\to\!-2.9$ pp); GRPO-MT
  partly mitigates ($-1.1$); PIRL restores it to base level ($\approx\!0$)
  under the full PIRL configuration. The effect is milder on the more robust
  Qwen3-VL-8B.}
  \label{fig:mean_tstress_gap}
\end{figure}

\par\bigskip

\section{Supplementary Discussion}
\label{sec:appendix_discussion}

This section provides additional discussion to complement the findings in the
main content, offering nuanced interpretations of the prompt-robustness gap.

\begin{itemize}
\item \textbf{RLVR optimization can systematically impair prompt-level robustness.}
While RL-tuned models (e.g., Qwen2.5-VL) can outperform their base
versions on static VQA benchmarks~\citep{shao2024deepseekmath}, they can be
more fragile under distribution shift. A standard GRPO-tuned model on the Lingshu
medical VQA dataset~\citep{xu2025lingshu} drops about $10\%$ under dynamic evaluation, a larger degradation than the non-RL base model. This is consistent with the mechanism in
Proposition~\ref{prop:nonid_main}: unconstrained optimization on
$P_{\text{train}}$ can allow the policy to occupy a non-identifiable corner of
the policy space that under-performs on the rest of $\llbracket \bm{x}\rrbracket$.

\item \textbf{Failure can be amplified under rigorous stress testing.}
Fragility is amplified under stressed evaluation settings, where the prompt
distribution during test differs from training
(\maintable{1}). The test-mass scaling in
Proposition~\ref{prop:nonid_main} is consistent with this pattern when the
stress distribution places substantial mass on separated unseen prompt formats.

\item \textbf{Reasoning shortcuts vs.\ latent capacity.}
Our analysis suggests RLVR optimization can favor a narrow slice of the prompt
manifold that maximizes reward efficiency. In the formal model, this corresponds
to an out-of-distribution generalization failure: $P_{\text{train}}$
is concentrated on a specific template, while $P_{\text{test}}$
spreads mass over unseen prompt transformations. The policy overfits to prompt syntax rather than learning
the underlying reasoning ability~\citep{yue2025does,lu2025r}.
\end{itemize}

\end{document}